\documentclass[letterpaper]{article}

\usepackage[preprint]{aaai2027}

\usepackage[hyphens]{url}
\usepackage{graphicx}
\usepackage{natbib}
\usepackage{caption}
\usepackage{algorithm}
\usepackage{algpseudocode}
\usepackage{amsthm}
\usepackage{amssymb}
\usepackage{amsmath}
\usepackage{mathtools}
\usepackage{subcaption}
\usepackage{booktabs}
\usepackage{multirow}
\usepackage{newfloat}
\usepackage{listings}
\DeclareCaptionStyle{ruled}{labelfont=normalfont,labelsep=colon,strut=off}
\floatstyle{ruled}
\newfloat{listing}{tb}{lst}{}
\floatname{listing}{Listing}

\usepackage{docmute}

\makeatletter
\let\arxiv@original@xfloat\@xfloat
\def\@xfloat#1[#2]{%
  \def\arxiv@float@opt{#2}%
  \def\arxiv@float@H{H}%
  \ifx\arxiv@float@opt\arxiv@float@H
    \arxiv@original@xfloat{#1}[t]%
  \else
    \arxiv@original@xfloat{#1}[#2]%
  \fi
}
\makeatother

\newtheorem{theorem}{Theorem}
\newtheorem{lemma}{Lemma}
\newtheorem{remark}{Remark}

\newtheorem{property}{Property}

\newcommand{\OPT}{\mathrm{OPT}}
\newcommand{\ADV}{\mathrm{ADV}}
\def\delay{\texttt{delay}}
\def\service{\texttt{service}}
\def\cost{\texttt{cost}}

\title{Learning-Augmented and Randomized Algorithms\\ for Line Aggregation with Delays}
\author{
Tianhang Lu,
Runtian Ren,
Shengcai Liu,
Ke Tang
}
\affiliations{
Guangdong Provincial Key Laboratory of Brain-Inspired Intelligent Computation,\\
Department of Computer Science and Engineering,\\
Southern University of Science and Technology, Shenzhen 518055, China\\
\{liusc3,tangk3\}@sustech.edu.cn
}

\begin{document}

% Use a single bibliography at the end of this combined arXiv file.
\let\arxivOriginalBibliography\bibliography
\renewcommand{\bibliography}[1]{}

% Main paper: the original file is kept unchanged.

\maketitle

\begin{abstract}
This paper studies learning-augmented and randomized online aggregation with delays on a line metric. 
We consider advice given as online suggested service lengths, and evaluate the algorithms in terms of robustness and consistency. 
For each $\lambda \in (0,1]$, we first propose a deterministic learning-augmented \textsc{Balance} algorithm that is $(4/\lambda+1/\lambda^2)$-robust and $(4+\lambda)$-consistent. 
We also propose a randomized algorithm for the problem in the classical adversarial model, which is $(e+1)$-competitive against an oblivious adversary, improving over the deterministic $5$-competitive \textsc{Balance} benchmark~\cite{bienkowski2013chain}.
Notably, this competitive ratio is even lower than the lower bound of $4$ for deterministic online algorithms.
Moreover, we establish a lower bound of $e$ on the competitive ratio of randomized online algorithms, improving the previous lower bound of $e/(e-1)$.
Besides, we combine the two ideas and obtain a randomized learning-augmented algorithm that is $(e/\lambda+1/\lambda^2)$-robust and $(e+\lambda)$-consistent. 
Finally, we conduct numerical experiments to complement our theoretical analysis and evaluate the empirical performance of our algorithms.
\end{abstract}

\section{Introduction}\label{sec:intro}
Aggregation captures a basic tradeoff in online decision making: serving requests early reduces waiting cost, while delaying service may allow many requests to be served together.
This tradeoff appears in inventory replenishment \cite{askoy1988multi, goyal1989joint, joneja1990joint, sindhuchao2005integrated, khouja2008review}, logistics \cite{brahimi2006single, jans2008modeling, quadt2008capacitated, bushuev2015review, karimi2003capacitated}, data transmission \cite{yuan2003synchronization, leung2007overview} etc, and has been studied under the broader framework of multi-level aggregation \cite{bienkowski2020mlap, mari2024online}.
In this paper, we focus on the line-metric case.
Requests arrive over time at different positions on a line that starts from a fixed root.
At any time, the algorithm may choose how far to send a service from the root.
A longer service incurs a larger cost, but it can clear all pending requests located within its reach.
Requests that are not yet served continue to wait and accumulate delay cost over time.
Note that even this line case already contains the two main decisions made online: when to serve, and how far the service should extend to.

Worst-case competitive analysis yields performance guarantees under adversarial arrival sequences.
The classical benchmark for Line Aggregation with Delays (LAD) is the deterministic \textsc{Balance} algorithm of Bienkowski et al.~\shortcite{bienkowski2013chain}.
The algorithm uses a grid in which the length of each level is a power of $2$ and serves a level whenever its accumulated delay reaches one quarter of its length.
The resulting $5$-competitive ratio is the best deterministic upper bound currently known.
Together with the lower bound $4$ \cite{bienkowski2020mlap}, this implies the optimal deterministic competitive ratio is within $[4,5]$.

Although online algorithms are robust in the worst case, they are often overly cautious under uncertainty.
This caution can lead to poor average performance in many real-world settings.
On the other hand, machine learning methods offer a complementary approach.
They use historical data to build predictive models, statistically achieving strong average performance in practice.
However, their predictions may be unreliable under distribution shift and may also fail when test instances are chosen adversarially~\cite{ovadia2019can,koh2021wilds}.
This motivates the learning-augmented method, which aims to combine the worst-case robustness of online algorithms with the predictive power of machine learning methods.
In the learning-augmented setting, algorithms are typically evaluated according to two standard criteria: \textbf{robustness} and \textbf{consistency}.
Robustness measures the worst-case competitive ratio under predictions.
Consistency measures the improved performance guarantee when the predictions are accurate.

Our first result is a deterministic learning-augmented algorithm for LAD.
Rather than blindly following the advice, we use the advice to change the urgency of requests.
Before a request is covered by the advice, its urgency grows slowly.
After the advice has covered it, urgency grows faster.
Specifically, for each $\lambda \in (0,1]$, the algorithm is $(4/\lambda+1/\lambda^2)$-robust and $(4+\lambda)$-consistent.
Here, $\lambda$ is a tradeoff parameter that represents how much the algorithm trusts the advice.

\begin{table*}[t]
\centering
\caption{
Comparison between the best known classical bounds and our results for LAD.
}
\label{tab:comparison}
\small
\begin{tabular}{@{}lccc@{}}
\toprule
\multirow{2}{*}{Algorithm type}
& \multicolumn{2}{c}{Classical online algorithms}
& \multirow{2}{*}{Learning-augmented online algorithms} \\
\cmidrule(lr){2-3}
& Upper bound & Lower bound & \\
\midrule
Deterministic
& $5$~\cite{bienkowski2013chain}
& $4$~\cite{bienkowski2020mlap}
& $\left(\frac{4}{\lambda}+\frac{1}{\lambda^2}\right)$-robustness
  and $(4+\lambda)$-consistency~(this paper) \\[3pt]

Randomized
& $e+1$~(this paper)
& $e$~(this paper)
& $\left(\frac{e}{\lambda}+\frac{1}{\lambda^2}\right)$-robustness
  and $(e+\lambda)$-consistency~(this paper) \\
\bottomrule
\end{tabular}
\end{table*}

Our second result concerns randomization in the classical adversarial model.
We show that the loss caused by the fixed grid in \textsc{Balance} can be reduced by randomizing the grid.
Specifically, we replace the fixed grid with a randomly shifted grid to design a randomized algorithm with competitive ratio $e+1$.
This competitive ratio is even lower than the lower bound of $4$ for deterministic online algorithms.
We also study the limitations of randomized algorithms.
When all requests have the same length, LAD reduces to the TCP acknowledgment problem.
Thus the known lower bound for TCP acknowledgment implies an $e/(e-1)$ lower bound for randomized algorithms for LAD.
We construct a new hard instance for LAD, improving this lower bound to $e$.

Finally, we combine the above two ideas. 
Our randomized learning-augmented algorithm uses the same randomly shifted grid, but it is built from advice-biased request urgency.
For each $\lambda \in (0,1]$, the algorithm is $(e/\lambda+1/\lambda^2)$-robust and $(e+\lambda)$-consistent.
The results show that learning augmentation and randomization are compatible: advice improves the performance when it is useful, while randomization improves the underlying adversarial benchmark.
We summarize our main contributions as follows.
Table~\ref{tab:comparison} compares our bounds with the best-known bounds for LAD.
\begin{enumerate}
\item We propose a deterministic learning-augmented online algorithm for LAD, with $(4/\lambda+1/\lambda^2)$-robustness and $(4+\lambda)$-consistency.
\item We propose an $(e+1)$-competitive randomized algorithm for LAD in the adversarial model, and establish an $e$ lower bound on the competitive ratio of randomized algorithms.
\item We propose a randomized learning-augmented online algorithm for LAD, with $(e/\lambda+1/\lambda^2)$-robustness and $(e+\lambda)$-consistency.
\item We conduct numerical experiments to evaluate the performance of our algorithms.
These results demonstrate that our algorithms outperform existing online algorithms and remain stable as the advice replacement rate increases.
\end{enumerate}

\section{Related Work}\label{sec:related-work}
Learning-augmented algorithms seek to combine worst-case algorithmic guarantees with possibly noisy predictions on future inputs.
This framework has been extensively studied for a range of fundamental problems, including paging~\cite{lykouris2021competitive}, bidding~\cite{angelopoulos2026learning} and matching~\cite{antoniadis2020secretary,choo2026learning}.
See the survey~\cite{mitzenmacher2022algorithms} for a broader overview.
In the following, we review prior work on classical and learning-augmented algorithms for the multi-level aggregation problem and the TCP acknowledgment problem.

\paragraph{Multi-Level Aggregation Problem.}
Aggregation with delays is commonly studied through the multi-level aggregation problem~\cite{bienkowski2020mlap, azar2019general, bienkowski2021new, mari2024online}, where requests arrive at nodes of a rooted tree and a service corresponds to a rooted subtree.  
The depth-one case captures TCP acknowledgment \cite{dooly2001tcp, karlin2001dynamic, seiden2000guessing}, and the depth-two case captures the joint replenishment problem~\cite{buchbinder2008online, bienkowski2014better}. 
For general tree metrics, poly-logarithmic and depth-dependent competitive algorithms are known for delay and deadline variants~\cite{azar2019general,buchbinder2017depth}.
Recently, one work studies learning-augmented algorithms for non-clairvoyant joint replenishment with deadlines~\cite{dinitz2025jrpdeadlines}, while other studies consider cases with non-monotone delays~\cite{azar2026beyond, azar2026online, moseley2025putting, ezra2026universal, bhore2026online}.

\paragraph{TCP Acknowledgment.}
The TCP acknowledgment problem was introduced by Dooly, Goldman, and Scott~\shortcite{dooly2001tcp} as a canonical model for balancing acknowledgment cost against packet latency.
They gave a $2$-competitive deterministic algorithm and a matching deterministic lower bound.
Randomization improves the best possible guarantee.
Seiden~\shortcite{seiden2000guessing} developed a guessing-game framework that underlies the $e/(e-1)$ randomized lower bound.
Karlin, Kenyon, and Randall~\shortcite{karlin2001dynamic} gave a matching $e/(e-1)$-competitive randomized algorithm.
Learning-augmented variants of TCP acknowledgment have been studied more recently.
Bamas, Maggiori, and Svensson~\shortcite{bamas2020primaldual} applied a primal-dual learning-augmented framework in which the prediction is a feasible solution, or equivalently action advice.
Im et al.~\shortcite{im2023online} later considered predictions of packet arrivals and introduced a temporal prediction-error measure.
Their algorithms achieve near-optimal consistency while preserving robustness close to the best classical guarantees.

\paragraph{Line Aggregation with Delays.}
The line aggregation with delays was introduced by Khanna et al.~\shortcite{khanna2002control}, who gave an $O(\log\alpha)$-competitive online algorithm, where $\alpha$ is the sum of all edge lengths.
Subsequently, Brito et al.~\shortcite{brito2004competitive} presented an 8-competitive online algorithm.
Bienkowski et al.~\shortcite{bienkowski2013chain} later proposed a 5-competitive algorithm named \textsc{Balance} and established a deterministic lower bound of $2+\phi$ with $\phi=(1+\sqrt{5})/2$ denoting the golden ratio.
Subsequent work on multi-level aggregation strengthened the deterministic lower bound for the line case to $4$~\cite{bienkowski2020mlap}.

\section{Preliminaries}\label{sec:pre}
\paragraph{Line Aggregation with Delays.}
An instance $I$ consists of a finite sequence of requests on the half-line $\mathbb{R}_+$ rooted at $0$.
Each request is a tuple $r=(\tau,x,w)$, where $\tau$ is its arrival time, $x\in \mathbb{R}_{++}$ is its location, and $w\in \mathbb{R}_{++}$ is its delay rate.
After arriving, a request remains pending until it is served.
If request $r$ is served at time $t$, it incurs a delay cost of $w(t-\tau)$.
A solution is represented by a set of time--location pairs $X=\{(t_1,x_1),\ldots,(t_n,x_n)\}$.
Here, $(t_i,x_i)$ denotes a service performed at location $x_i$ at time $t_i$, and the cost of this service is $x_i$.
The solution $X$ is feasible if, for every request $r=(\tau,x,w)$, there exists a pair $(t_i,x_i)\in X$ such that $t_i\geq\tau$ and $x_i\geq x$.
The delay cost of $X$ is the sum of the delay costs incurred by all requests, whereas its service cost is the sum of the costs of all services in $X$.
The total cost of $X$ is the sum of these two quantities.
For any algorithm $\mathcal{A}$ that produces a feasible solution, let $\delay(\mathcal{A})$, $\service(\mathcal{A})$, and $\cost(\mathcal{A})$ denote the delay cost, service cost, and total cost of its output, respectively.

\paragraph{Advice.}
The algorithm receives advice online in the form of suggested service lengths, which may be provided at arbitrary times.
At any time $t$, the advice may provide a length $a_t>0$.
For a request $r=(\tau,x,w)$, the advice \emph{covers} $r$ at time $t\geq\tau$ if $a_t\geq x$.
We denote the first such time by $t_{\text{cov}}(r)$ and call it the \emph{first coverage time} of $r$.
We call the advice \emph{feasible} if it eventually covers every request, and assume this condition throughout this paper.

\paragraph{Performance Guarantees.}
For each instance $I$, let $\OPT$ and $\ADV$ denote an optimal offline solution and the solution prescribed by the advice, respectively.
A (randomized) online algorithm $\mathcal{A}$ is called \textit{$r$-competitive} or \textit{$r$-robust} if $\mathbb{E}[\cost(\mathcal{A})] \le r \cdot \cost(\OPT)$ for any instance $I$.
Besides, $\mathcal{A}$ is called \textit{$c$-consistent} if $\mathbb{E}[\cost(\mathcal{A})] \le c \cdot \cost(\ADV)$ for any instance $I$ and any feasible advice $\ADV$.
For deterministic algorithms, the expectations are omitted.
Robustness measures protection against any feasible advice, whereas consistency measures performance when the advice is reliable.

\paragraph{The \textsc{Balance} Algorithm.}
The best deterministic online algorithm currently known for LAD is \textsc{Balance} \cite{bienkowski2013chain}.
The algorithm uses a grid in which the length $g$ of each level is a power of $2$, that is, $g=2^j$.
For each level, let $W(t,g)$ denote the total delay accumulated by pending requests in $[0,g]$ up to time $t$.
Whenever one or more levels satisfy $W(t,g)=g/4$, \textsc{Balance} serves the largest such level at this moment.
The following theorem describes the competitive ratio of \textsc{Balance}.

\begin{theorem}\cite{bienkowski2013chain}
\textsc{Balance} is $5$-competitive for LAD.
\end{theorem}

\section{Learning-Augmented Deterministic Algorithm}\label{sec:la-d}
In this section, we present a learning-augmented algorithm for LAD.
Our learning-augmented online algorithm achieves both robustness and consistency.
Before describing the algorithm, we introduce two key concepts that are central to the subsequent algorithm design and analysis: the request urgency and the level potential.

\begin{algorithm}[t]
\caption{Learning-Augmented \textsc{Balance} (LA-B)}
\label{alg:la-balance}
\begin{algorithmic}[1]
\Require A trust parameter $\lambda\in(0,1]$
\State $P\gets\emptyset$ \Comment{set of pending requests}
\While{an event occurs at time $t$ that is, at least one request arrives, or the advice provides a length, or $Q(t,g)=g/4$ for some level $g$}
\ForAll{requests $r=(\tau,x,w)$ with $\tau=t$}
    \State Add $r$ to $P$ with $\mathrm{status}(r)\gets\text{\textsc{Uncovered}}$.
\EndFor
\If{the advice provides a length $a_t>0$}
    \ForAll{$r=(\tau,x,w)\in P$ with $x\le a_t$ and $\mathrm{status}(r)=\text{\textsc{Uncovered}}$}
        \State $\mathrm{status}(r)\gets(\text{\textsc{Covered}},t)$
    \EndFor
\EndIf
\If{there exists a level $g$ with $Q(t,g)=g/4$}
    \State $g^\star\gets\max\{g=2^j: Q(t,g)=g/4\}$
    \State Perform a service of length $g^\star$ at time $t$.
    \State $P\gets P\setminus\{(\tau,x,w)\in P:x\le g^\star\}$
\EndIf
\EndWhile
\end{algorithmic}
\end{algorithm}

\paragraph{Request Urgency and Level Potential.}
Fix a trust parameter $\lambda\in(0,1]$.
For each request $r=(\tau,x,w)$, define its \emph{request urgency} at time $t$ as
$$
q(t,r)=\lambda w\bigl(\min\{t,t_{\text{cov}}(r)\}-\tau\bigr)_+{}+\frac{1}{\lambda}w(t-t_{\text{cov}}(r))_+,
$$
where $(z)_+=\max\{z,0\}$.
Request urgency assigns a smaller weight to the delay cost accumulated before the advice first covers $r$ and a larger weight to that accumulated after the first coverage time $t_{\mathrm{cov}}(r)$.
The coverage time $t_{\text{cov}}(r)$ provides a temporal signal about when $r$ should be served.
Before time $t_{\text{cov}}(r)$, the algorithm can afford to wait for the service suggested by the advice.
After $t_{\text{cov}}(r)$, additional waiting should make $r$ more urgent to serve.
The trust parameter $\lambda$ controls how strongly this signal affects the request urgency.
In particular, when $\lambda=1$, the request urgency equals the actual delay cost accumulated by $r$.
A smaller $\lambda$ places greater trust in the advice by reducing the urgency accumulated before $t_{\text{cov}}(r)$ and increasing the urgency accumulated afterward.
Nevertheless, for each request $r=(\tau,x,w)$, the request urgency remains comparable to the actual delay cost:
$$
\lambda w(t-\tau)\le q(t,r)\le\frac{1}{\lambda}w(t-\tau).
$$

Note that a service of length $g$ clears all pending requests in $[0,g]$.
We define the \emph{level potential} of $g$ at time $t$ as
$$
Q(t,g)=\sum_{\substack{r=(\tau,x,w)\text{ pending at }t\\x\le g}}q(t,r).
$$
Thus, $q(t,r)$ measures the urgency of an individual request, whereas $Q(t,g)$ aggregates the urgencies of all pending requests that would be cleared by a service of length $g$.
When $\lambda=1$, the request urgencies equal the actual accumulated delay costs, and hence $Q(t,g)=W(t,g)$.

\paragraph{Algorithm Description.}
We present Learning-Augmented \textsc{Balance} (LA-B) in Algorithm~\ref{alg:la-balance} and describe how it works.
A level of length $g$ is called \emph{tight} when $Q(t,g)=g/4$.
The main idea of LA-B is to serve the largest tight level.
Whenever one or more levels become tight, LA-B chooses the largest one, $g^\star$, and performs a service of length $g^\star$.

Algorithm~\ref{alg:la-balance} implements this rule through an event-driven procedure.
After initializing the pending set $P$, the algorithm handles three types of events.
When a request arrives, it is added to $P$ as an uncovered request.
When the advice provides a length $a_t>0$, each uncovered request in $P$ with location at most $a_t$ is marked as covered, and $t$ is recorded as its first coverage time.
When one or more levels become tight, the algorithm selects the largest one, serves it, and removes all requests in that level from $P$.
We remark that the algorithm uses each advice length $a_t$ to update the coverage status of pending requests and records their first coverage times.
This information determines $q(t,r)$ and hence $Q(t,g)$.

\begin{theorem}\label{thm:la-balance-main}
For each $\lambda\in(0,1]$, LA-B is $(4/\lambda+1/\lambda^2)$-robust and $(4+\lambda)$-consistent for LAD.
\end{theorem}

\begin{proof}[Proof sketch]
Fix an instance $I$.
When LA-B serves length $g$ at time $t$, tightness implies $Q(t,g)=g/4$.
Since $q(t,r)\ge\lambda w(t-\tau)$ for each request $r$ cleared by this service, their total delay cost is at most $g/(4\lambda)$.
Thus, the service and delay costs associated with this step sum to at most $g+g/(4\lambda)$.

For robustness, we apply the standard charging argument for \textsc{Balance} to the grid.
For each LA-B service of length $g$, we identify a portion of the cost of $\OPT$ with value at least $\lambda g/4$.
The factor $\lambda$ follows from $q(t,r)\le w(t-\tau)/\lambda$, which implies that the real delay corresponding to any amount of urgency is at least $\lambda$ times that amount.
Moreover, the charged portions can be chosen disjointly across all LA-B services.
Consequently, the robust ratio is at most
$$
\frac{g+g/(4\lambda)}{\lambda g/4}=\frac{4}{\lambda}+\frac{1}{\lambda^2}.
$$

For consistency, partition the delay incurred by each request $r=(\tau,x,w)$ into the portions incurred before and after its first coverage time.
Let $t_{\mathrm{ser}}(r)$ denote the time at which LA-B serves $r$, and define 
$$
\text{pre}(r)=w(\min\{t_{\text{ser}}(r),t_{\text{cov}}(r)\}-\tau)_+
$$ 
and 
$$
\text{post}(r)=w(t_{\text{ser}}(r)-t_{\text{cov}}(r))_+.
$$
Set $\mathrm{Pre}=\sum_r\operatorname{pre}(r)$ and $\mathrm{Post}=\sum_r\operatorname{post}(r)$.
These two components satisfy $\mathrm{Pre}\leq \delay(\ADV)$ and $\mathrm{Post}\leq \lambda \cdot\service(\ADV)$.
At each LA-B service of length $g$ at time $t$, tightness gives $g=4Q(t,g)$.
Since this service removes exactly the pending requests contributing to $Q(t,g)$, summing over all services yields $$\service(\text{LA-B})=4\sum_r q(t_{\mathrm{ser}}(r),r).$$
Note that $q(t_{\text{ser}}(r),r)=\lambda \cdot \text{pre}(r)+\frac{1}{\lambda} \cdot \text{post}(r)$.
Therefore, $\service(\text{LA-B})=4\lambda \cdot \mathrm{Pre}+\frac{4}{\lambda}\cdot \mathrm{Post}$ is at most
$$
4\lambda\cdot\delay(\ADV)+4\cdot\service(\ADV).
$$
Furthermore, $\delay(\text{LA-B})=\mathrm{Pre}+\mathrm{Post}$ is at most
$$
\delay(\ADV)+\lambda \cdot \service(\ADV).
$$
We can thus upper bound $\cost(\text{LA-B})$ by
$$
(4\lambda+1)\cdot\delay(\ADV)+(4+\lambda)\cdot\service(\ADV).
$$
This is no more than $(4+\lambda)\cdot \cost(\ADV)$.
Hence, the theorem follows.
\end{proof}

\section{Randomized Algorithm}\label{sec:random}
In this section, we give an $(e+1)$-competitive randomized algorithm and an $e$ lower bound against oblivious adversaries.
We first describe the randomized algorithm.

\begin{algorithm}[t]
\caption{\textsc{Randomized Balance} (RB)}
\label{alg:rb}
\begin{algorithmic}[1]
\State Sample $U\sim\mathrm{Unif}[0,1]$ and set $b_i=e^{i+U}$ for all $i\in\mathbb Z$.
\While{some level $b_i$ satisfies $W(t,b_i)=b_i/e$}
 \State $j\gets\max\{i:W(t,b_i)=b_i/e\}$.
 \State Serve length $b_j$ to clear all pending requests with $x_r\le b_j$.
\EndWhile
\end{algorithmic}
\end{algorithm}

\paragraph{Algorithm Description.}
We present Randomized \textsc{Balance} (RB) in Algorithm~\ref{alg:rb} and describe how it works.
At the beginning, we sample $U\sim\mathrm{Unif}[0,1]$ and define the levels $b_i=e^{i+U}$ for all $i\in\mathbb Z$.
A level $b_i$ is called \emph{tight} when $W(t,b_i)=b_i/e$.
Whenever one or more levels become tight, RB serves the largest such level at this moment.

\begin{theorem}\label{thm:rb-main}
RB is a randomized $(e+1)$-competitive algorithm against an oblivious adversary for LAD.
\end{theorem}

\begin{figure*}[t]
\centering
\includegraphics[width=0.8\textwidth]{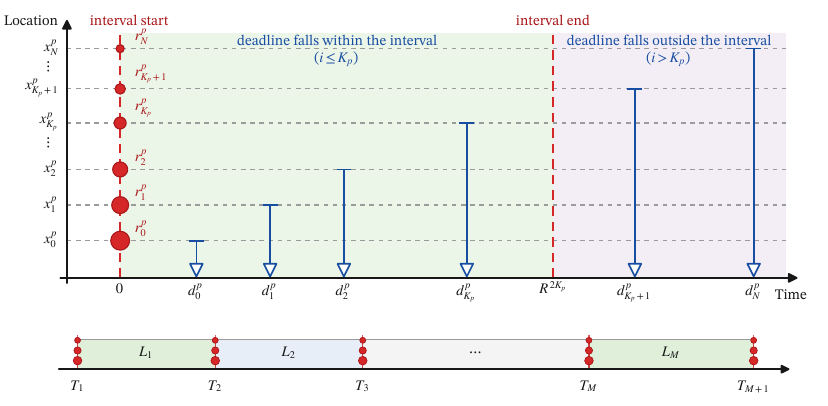}
\caption{Illustration of the hard instance. Top: request locations and relative deadlines in interval $p$. Bottom: the $M$ consecutive intervals forming the full instance.}
\label{fig:random-lower-hard-instance}
\end{figure*}

\begin{proof}[Proof sketch]
Fix an instance $I$.
When RB serves length $b_j$ at time $t$, tightness gives $W(t,b_j)=b_j/e$, so the service and delay costs associated with this step sum to $b_j+b_j/e=(e+1)b_j/e.$
It thus suffices to prove that the expected total base charge is at most the offline optimum, i.e.,
\begin{equation}\label{eq:ran-charge}
\mathbb E_U\!\left[\sum_{\text{\rm RB services }(t,b_j)}\frac{b_j}{e}\right]\le\cost(\OPT).
\end{equation}

We apply the standard charging argument for \textsc{Balance} to the randomly shifted grid.
For each RB service $(t,b_j)$, the argument associates it with an $\OPT$ service of length $A>0$.
Let $H_U(A)=\min\{i\in\mathbb Z:b_i>A\}$ be the index of the first randomly shifted level that strictly exceeds $A$, and set $h=\min\{j,H_U(A)\}$.
If no $\OPT$ service is associated with the RB service, we use the convention $b_h=0$.
We split its base charge as $$b_j/e=(b_j-b_h)/e+b_h/e.$$

We first charge $(b_j-b_h)/e$ to the delay cost of $\OPT$.
The charging argument guarantees that the requests served by RB from $(b_h,b_j]$ at time $t$ remain unserved by $\OPT$ until $t$.
Moreover, tightness at $b_j$ and the level invariant at $b_h$ give
$$
W(t,b_j)-W(t,b_h)\ge\frac{b_j-b_h}{e}.
$$
Therefore, $\OPT$ incurs delay at least $(b_j-b_h)/e$ on these requests.
The charged request sets are disjoint across RB services because each request is cleared by RB only once.
Hence the total delay charge is at most $\delay(\OPT)$.

We then charge the remaining amount $b_h/e$ to the associated $\OPT$ service of length $A$.
The charging argument guarantees that $h\le H_U(A)$ and that the same $\OPT$ service is charged at most once at each level $h$.
Thus, for a fixed offset $U$, its total service charge is at most $\sum_{i\le H_U(A)}b_i/e$.
Since the levels form a randomly shifted geometric sequence,
$$
\mathbb E_U\!\left[\sum_{i\le H_U(A)}\frac{b_i}{e}\right]=\frac{A}{e-1}\int_0^1 e^z\,dz=A.
$$
Thus the expected service charge assigned to every $\OPT$ service is at most its service cost.
Summing the delay and service charges yields (\ref{eq:ran-charge}).
Multiplying by the local factor $e+1$ completes the proof.
\end{proof}

\begin{remark}
We explain why $e$ is the optimal choice for the base of the grid.
Generally, for any $c>1$, consider a variant of Algorithm~\ref{alg:rb} obtained by replacing $b_i=e^{i+U}$ with $b_i=c^{i+U}$.
The expected charge assigned to this service is at most
$$
\mathbb E_U\!\left[\sum_{i\le H_U(A)}\frac{b_i}{c}\right]=\frac{A}{\ln c}.
$$
%This charge is bounded by $A$ only if $\ln c\ge 1$ (i.e. $c\ge e$).
Meanwhile, each tight RB service incurs a cost $b_j+b_j/c=(c+1)b_j/c$.
The competitive ratio is thus bounded by $(c+1)\max\{1,1/\ln c\}$, which is minimized to $e+1$ when $c = e$.
\end{remark}

\paragraph{Lower Bound.}
We next show that the improvement is intrinsic to LAD rather than inherited from the single-location TCP acknowledgment problem.

\begin{theorem}\label{thm:random-lower-main}
No online randomized algorithm can achieve a competitive ratio less than $e$ against an oblivious adversary for LAD.
\end{theorem}

\begin{proof}[Proof sketch]
By Yao's principle~\shortcite{yao1977probabilistic}, it suffices to construct a distribution under which every deterministic algorithm has a competitive ratio approaching $e$.
The hard instance $I$ is constructed as follows.
Fix $\varepsilon\in(0,1]$, an integer $N$ satisfying $(N+1)\varepsilon\geq 1$, and $R>\sqrt e$.
Draw $K_1,\ldots,K_M$ independently from a distribution on $\{0,\ldots,N\}$ whose tail probabilities satisfy $\Pr[K_p\geq i]=e^{-i\varepsilon}$ for $i=0,\ldots,N$.
Starting from $T_1=0$, define $L_p=R^{2K_p}$ and $T_{p+1}=T_p+L_p$ for each $p=1,\ldots,M$.
Thus, $I$ consists of $M$ consecutive time intervals $[T_p,T_{p+1}]$, with the $p$-th interval having length $L_p$.
Requests arrive only at the beginning of each interval.
Each interval $[T_p,T_{p+1}]$ receives requests $r^p_i=(T_p,x^p_i,w^p_i)$ for $i=0,\ldots,N$, where $x^p_i=e^{i\varepsilon}$ and weight $w^p_i=R^{1-2i}$.
Define the (relative) deadline of request $r^p_i$ as $d^p_i=R^{2i}/2$.
Since the $p$-th interval has length $R^{2K_p}$, the deadline of request $r^p_i$ falls within this interval if and only if $i\leq K_p$.
Figure~\ref{fig:random-lower-hard-instance} summarizes both one interval and the total time.

Consider a deterministic algorithm for $I$ satisfying the following property.
\begin{property}\label{prop:deadline}
For each interval $[T_p,T_{p+1}]$, if the deadline of $r^p_j$ falls within this interval, i.e., $j\leq K_p$, then the algorithm serves $r_j^p$ no later than time $T_p+d_j^p$.
\end{property}
Denote $G_{N,\varepsilon}=e^{1-\varepsilon}(N+1)\varepsilon$.
We first prove that $\mathbb{E}[\cost(\mathcal{A})]\geq MG_{N,\varepsilon}$ for every algorithm $\mathcal{A}$ satisfying Property~\ref{prop:deadline}.
More generally, for any algorithm $\mathcal{A}$ that may not satisfy Property~\ref{prop:deadline}, we construct an auxiliary algorithm $\mathcal{A}'$ that simulates $\mathcal{A}$.
At the first artificial deadline missed by $\mathcal{A}$ in interval $p$, algorithm $\mathcal{A}'$ immediately performs one additional service of length $x_N^p$.
Thus, $\mathcal{A}'$ performs at most one additional service in each interval and satisfies Property~\ref{prop:deadline}.

Let $\rho_p$ denote the probability that $\mathcal{A}$ misses the deadline $d^p_j$ for at least one $j$ during interval $p$.
If $\mathcal{A}$ misses the deadline $d^p_j$, then the resulting delay cost incurred by $\mathcal{A}$ is at least $w^p_jd^p_j$, which is equal to $R/2$.
Therefore,
\begin{equation}\label{eq:dcost-A}
\mathbb{E}[\cost(\mathcal{A})]\geq\frac{R}{2}\sum_{p=1}^M\rho_p.
\end{equation}
Since $\mathcal{A}'$ performs an additional service of length $x^p_N$ only when $\mathcal{A}$ misses a deadline in the $p$-th interval, and since $\mathcal{A}'$ satisfies Property~\ref{prop:deadline}, it follows that 
$$
\mathbb{E}[\cost(\mathcal{A})]+\sum_{p=1}^Mx^p_N\rho_p\geq \mathbb{E}[\cost(\mathcal{A}')]\geq MG_{N,\varepsilon}.
$$
Combining this inequality with inequality~(\ref{eq:dcost-A}), we obtain
\begin{equation}\label{eq:cost-A}
\mathbb{E}[\cost(\mathcal{A})]\geq\frac{RM}{R+2e^{N\varepsilon}}G_{N,\varepsilon}.
\end{equation}

For an upper bound on $\cost(\OPT)$, we consider an algorithm $\mathcal{B}$ such that $\mathcal{B}$ serves $x^p_{K_p}$ at each $T_p$, for $p=1,\ldots,M$, and serves $x^M_N$ at $T_{M+1}$.
The expected service cost in the $p$-th interval is $\mathbb E[x^p_{K_p}]=1+N(1-e^{-\varepsilon})$.
Therefore, $\mathbb{E}[\service(\mathcal{B})]$ is equal to
\begin{equation}\label{eq:scost-B}
\sum_{p=1}^M\mathbb E[x^p_{K_p}]+x^M_N=M+MN(1-e^{-\varepsilon})+e^{N\varepsilon}.
\end{equation}
On the other hand, $\mathcal{B}$ does not immediately serve any request $r_j^p$ with $j>K_p$.
We show that
\begin{equation}\label{eq:dcost-B}
\mathbb{E}[\delay(\mathcal{B})]\leq MNR\frac{1-e^{-\varepsilon}}{e^{-\varepsilon} R^2-1}.
\end{equation}
Since $\mathbb{E}[\cost(\OPT)]\leq\mathbb{E}[\cost(\mathcal{B})]$, combining~(\ref{eq:cost-A}), (\ref{eq:scost-B}), and~(\ref{eq:dcost-B}) with Yao's principle gives a lower bound on the competitive ratio.
For fixed $N$ and $\varepsilon$, letting $M\to\infty$ and then $R\to\infty$ makes this lower bound arbitrarily close to
\begin{equation}\label{eq:lower-bound}
\frac{e^{1-\varepsilon}(N+1)\varepsilon}{1+N(1-e^{-\varepsilon})}.
\end{equation}
Taking $N\to\infty$ and then letting $\varepsilon\to 0$, the lower bound in~\eqref{eq:lower-bound} converges to $e$.
Therefore no randomized online algorithm can achieve a competitive ratio below $e$ against an oblivious adversary.
\end{proof}

\section{Learning-Augmented Randomized Algorithm}\label{sec:la-random}
In this section, we present a learning-augmented randomized algorithm for LAD.
The algorithm combines the learning-augmented deterministic algorithm with the randomized algorithm.
We describe the algorithm as follows.

\paragraph{Algorithm Description.}
We present Learning-Augmented Randomized \textsc{Balance} (LA-RB) in Algorithm~\ref{alg:la-rb} and describe how it works.
Fix a trust parameter $\lambda\in(0,1]$.
At the beginning, we sample $U\sim\mathrm{Unif}[0,1]$ and define the levels $b_i=e^{i+U}$ for all $i\in\mathbb Z$.
For each $b_i$, let $Q(t,b_i)$ denote the level potential of the pending requests in $[0,b_i]$, computed using the advice-adjusted request urgency $q(t,r)$ introduced for LA-B.
A level is \emph{tight} when $Q(t,b_i)=b_i/e$.
Whenever one or more levels become tight, LA-RB serves the largest such level at this moment.
Thus, LA-RB uses advice to adjust the delay contributed by each request and uses the random levels to determine the available service lengths.
When $\lambda=1$, $Q(t,b_i)=W(t,b_i)$, and LA-RB reduces to RB.

The following theorem establishes the robustness and consistency of Algorithm~\ref{alg:la-rb}.
The proof combines the charging argument for LA-B with the randomly shifted grid analysis for RB.
We defer the details to the supplementary material.

\begin{algorithm}[t]
\caption{Learning-Augmented Randomized \textsc{Balance} (LA-RB)}
\label{alg:la-rb}
\begin{algorithmic}[1]
\Require A trust parameter $\lambda\in(0,1]$
\State Sample $U\sim\mathrm{Unif}[0,1]$ and set $b_i=e^{i+U}$ for all $i\in\mathbb Z$.
\State $P\gets\emptyset$ \Comment{set of pending requests}
\While{an event occurs at time $t$ that is, at least one request arrives, or the advice provides a length, or $Q(t,b_i)=b_i/e$ for some level $b_i$}
\ForAll{requests $r$ with $\tau_r=t$}
    \State Add $r$ to $P$ with $\mathrm{status}(r)\gets\text{\textsc{Uncovered}}$.
\EndFor
\If{the advice provides a length $a_t>0$}
    \ForAll{$r\in P$ with $x_r\le a_t$ and $\mathrm{status}(r)=\text{\textsc{Uncovered}}$}
        \State $\mathrm{status}(r)\gets(\text{\textsc{Covered}},t)$
    \EndFor
\EndIf
\If{there exists a level $b_i$ with $Q(t,b_i)=b_i/e$}
    \State $j\gets\max\{i:Q(t,b_i)=b_i/e\}$
    \State Perform a service of length $b_j$ at time $t$.
    \State $P\gets P\setminus\{r\in P:x_r\le b_j\}$
\EndIf
\EndWhile
\end{algorithmic}
\end{algorithm}

\begin{theorem}\label{thm:larb-main}
For each $\lambda\in(0,1]$, LA-RB is $(e/\lambda+1/\lambda^2)$-robust and $(e+\lambda)$-consistent for LAD against an oblivious adversary.
\end{theorem}

\section{Experiments}\label{sec:exp}
In this section, we conduct numerical experiments to evaluate the performance of our proposed algorithms (LA-B, RB, and LA-RB).
\footnote{The source code is included in the supplementary archive.}
% \footnote{Source code: https://anonymous.4open.science/r/LAD-77F5.}

\begin{figure*}[t]
\centering
\includegraphics[width=0.92\textwidth]{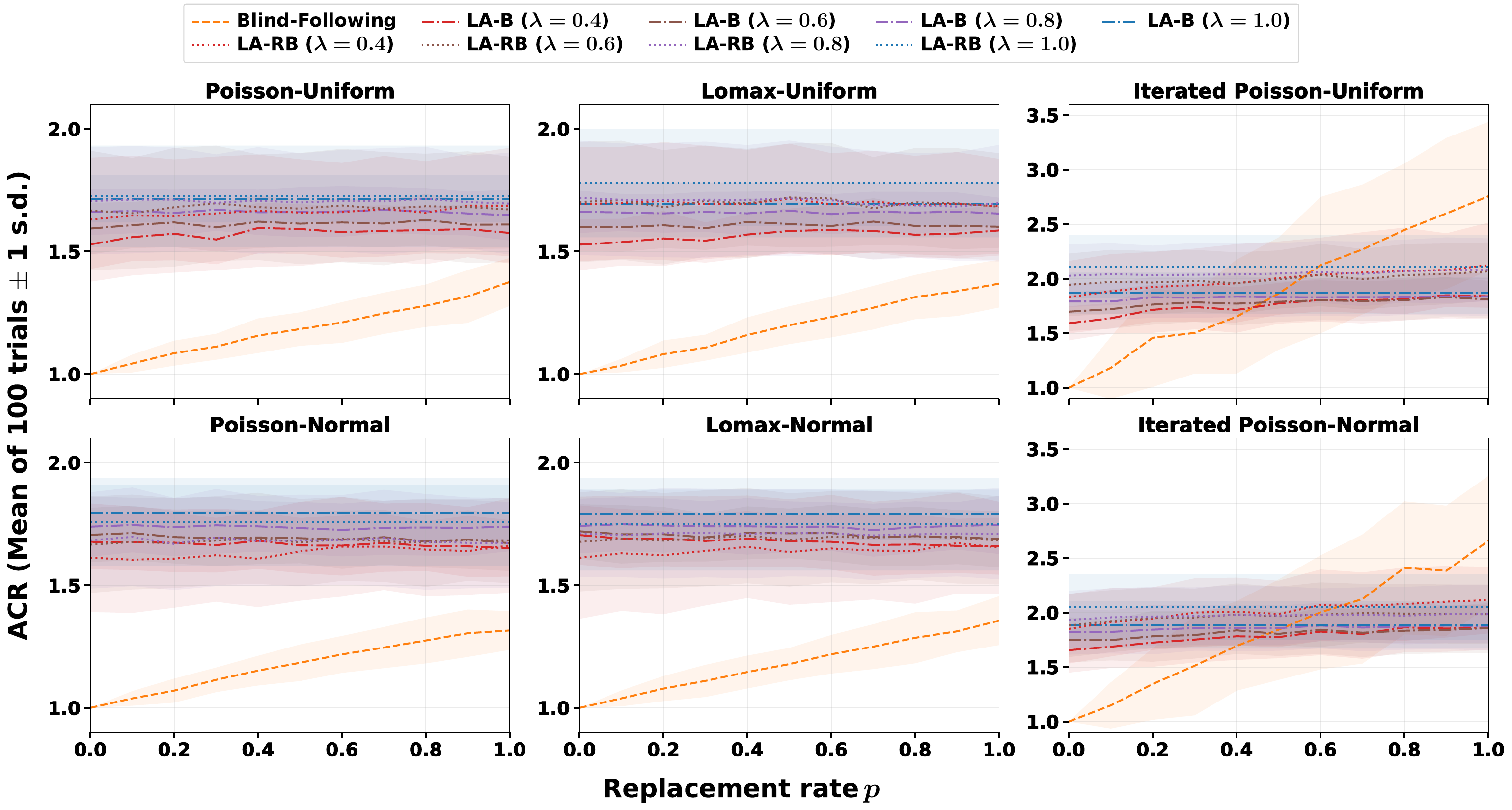}
\caption{The performance of algorithms under six distributions.}
\label{fig:exp}
\end{figure*}

\paragraph{Computational Settings.}
All experiments are conducted on a machine running Windows 11 with an AMD Ryzen 7 5800H CPU and 16 GB of memory.

\paragraph{Dataset.}
We set the time horizon to $T=100$ and the number of discrete locations to $N=100$, and sample each request's delay rate independently from $\operatorname{Unif}[0,1]$.
At each time--location pair $(t,x)$, the request count is sampled independently from a distribution $\mathcal{D}$ with location-dependent parameters.
We construct the distribution $\mathcal{D}$ as follows.
In the spatial domain, we consider two distributions over $x\in\{1,\ldots,N\}$: the uniform distribution, with $q_x=1/N$, and a discrete normal distribution,
$$q_x=\frac{\Phi\!\left(\frac{x+1/2-\mu}{\sigma}\right)-\Phi\!\left(\frac{x-1/2-\mu}{\sigma}\right)}{\Phi\!\left(\frac{N+1/2-\mu}{\sigma}\right)-\Phi\!\left(\frac{1/2-\mu}{\sigma}\right)},$$
where $\mu=(N+1)/2$, $\sigma=(N-1)/6$, and $\Phi$ is the standard normal cumulative distribution function.
In the temporal domain, for each spatial distribution, we consider three request-count distributions following Bamas et al.~\shortcite{bamas2020primaldual}: Poisson with mean $q_x$, Lomax with shape $2$ and scale $q_x$ (followed by unbiased randomized rounding to obtain integer request counts), and 10-times iterated Poisson initialized at $q_x$.
Note that, for all these distributions, the expected total number of requests at each time is $1$.

\paragraph{Advice.}
We perturb each instance with noise, and compute an optimal solution on this perturbed instance and use this as a prediction.
We fix a replacement rate $p\in[0,1]$.
For each time–location pair $(t,x)$ in an instance, we delete all requests at $(t,x)$ with probability $p$.
Independently, with probability $p$, we add a random count of new requests at $(t,x)$, with the count sampled from $\mathcal{D}$ and each delay rate independently sampled from $\operatorname{Unif}[0,1]$.
We then compute an optimal offline solution using the polynomial-time algorithm presented by Bienkowski et al.~\shortcite{bienkowski2013chain}.
To ensure feasibility, we add a service at the final request arrival to cover the farthest request left uncovered by the advice.
A similar generation procedure has also been used in prior work~\cite{bamas2020primaldual,grigorescu2022learning}, where it is referred to as the “replacement rate” strategy.

\paragraph{Algorithms.}
We evaluate LA-B and LA-RB with $\lambda\in\{0.4,0.6,0.8,1\}$.
At $\lambda=1$, LA-B and LA-RB reduce to Balance and RB, respectively.
We also include Blind-Following as a baseline, which directly adopts the advice as a solution to the original instance.

\paragraph{Results.}
For each of the six request distributions, we generate 100 independent random instances and evaluate all parameter configurations on the same set of instances.
The average competitive ratio (ACR) is defined as the average, over all trials, of the online algorithm’s total cost divided by the offline OPT cost.
As can be seen from Figure~\ref{fig:exp}, we make three main observations.
First, \textsc{LA-B} and \textsc{LA-RB} achieve lower ACRs than \textsc{Balance} and \textsc{RB}, respectively.
Second, the ACRs of \textsc{LA-B} and \textsc{LA-RB} change little as the replacement rate $p$ increases.
In contrast, the ACR of Blind-Following increases steadily and eventually exceeds those of both algorithms under the Iterated Poisson distributions.
Third, randomization is more beneficial under spatially non-uniform distributions.
At the same $\lambda$, \textsc{LA-RB} performs worse than \textsc{LA-B} under the uniform spatial distribution, but performs better under the discrete-normal spatial distribution with Poisson and Lomax arrivals, although not with Iterated Poisson arrivals.

\section{Conclusion}\label{sec:cnlu}
In this paper we study LAD in both learning-augmented and randomized models. 
We propose an $(e+1)$-competitive randomized algorithm and establish a lower bound of $e$ on the competitive ratio of randomized online algorithms. 
We also design deterministic and randomized learning-augmented online algorithms that achieve robustness and consistency.
One direction for future work is to close the gap of $1$ between the upper and lower bounds on the competitive ratios of deterministic and randomized online algorithms.
Another direction is to consider learning-augmented online aggregation on tree metrics and general metrics.
% which capture the joint replenishment problem (JRP)

\bibliography{aaai2027}


\appendix
\section*{Supplementary Material}

% Supplement-specific numbering: Figure S1, Table S1, Equation S1, Algorithm S1.
\setcounter{figure}{0}
\setcounter{table}{0}
\setcounter{equation}{0}
\setcounter{algorithm}{0}
\setcounter{listing}{0}
\renewcommand{\thefigure}{S\arabic{figure}}
\renewcommand{\thetable}{S\arabic{table}}
\renewcommand{\theequation}{S\arabic{equation}}
\renewcommand{\thealgorithm}{S\arabic{algorithm}}
\renewcommand{\thelisting}{S\arabic{listing}}

% The supplement has its own \maketitle in the standalone file.
% Suppress it here because we already inserted the supplement heading above.
\let\arxivOriginalMaketitle\maketitle
\def\maketitle{}

% Supplement: the original file is kept unchanged.

\title{Supplementary Material}

% Anonymous review version:
\author{
Anonymous Submission
}
\affiliations{
}

% Camera-ready version example:
% \author{
% First Author\textsuperscript{\rm 1},
% Second Author\textsuperscript{\rm 2}
% }
% \affiliations{
% \textsuperscript{\rm 1}Affiliation One\\
% \textsuperscript{\rm 2}Affiliation Two\\
% first@example.com, second@example.com
% }

\maketitle

\section*{Proof of Theorem 2}

\paragraph{Robustness.}
Fix a given instance $I$ and a feasible advice $\ADV$.
For each LA-B service, we first bound the cost of LA-B.

\begin{lemma}
\label{lem:one-service-cost}
If LA-B serves length $g$ at time $t$, then the service cost and the delay cost of the cleared requests sum to at most $g+g/(4\lambda)$.
\end{lemma}

\begin{proof}
The service clears all pending requests $r=(\tau,x,w)$ with $x\le g$.
The bound $q(t,r)\ge \lambda w(t-\tau)$ gives
$$
\sum_{\substack{r\text{ cleared}\\\text{at time }t}}w(t-\tau)
\le \frac{1}{\lambda}
\sum_{\substack{r\text{ cleared}\\\text{at time }t}}q(t,r)
=\frac{Q(t,g)}{\lambda}
=\frac{g}{4\lambda}.
$$
The last equality follows from tightness.
The service itself costs $g$.
These two bounds prove the lemma.
\end{proof}

We use the \textit{cover sequence} for \textsc{Balance}~\shortcite{bienkowski2013chain}.
Let $\ell(t)$ be the length of the LA-B service at time $t$.
For an LA-B service of length $g=2^j$ at time $t_0$, let $t_1$ be the last service before $t_0$, and recursively let $t_{i+1}$ be the last service before $t_i$ longer than $\ell(t_i)$.
Stop at the first $t_k$ with $\ell(t_k)\ge 2^j$.
Thus,
$$
t_k<t_{k-1}<\cdots<t_1<t_0.
$$
A zero-cost service of infinite length before the first arrival handles the boundary case.

Set the \textit{cover threshold} to $0$ on $(t_1,t_0]$ and to $\ell(t_i)$ on $(t_{i+1},t_i]$.
The thresholds do not decrease when moving backward, and every earlier LA-B service inside an interval has length at most its threshold.
Call an $\OPT$ service in $(t_k,t_0]$ \emph{large} if its length is at least the threshold at its service time.
Choose a maximum-length large service $(s^\star,L)$ if one exists; otherwise set $L=0$.
Define
$$
b=
\begin{cases}
\min\{j,\lfloor\log_2 L\rfloor+1\}, & L>0,\\
-\infty, & L=0,
\end{cases}
$$
and let $d_b=2^{b-2}$ for $b\in\mathbb Z$ and $d_{-\infty}=0$.
We also use the conventions $2^{-\infty}=0$ and $Q(t,0)=0$.

\begin{lemma}[Robustness]
\label{lemma:lab-robustness}
LA-B is $\left(4/\lambda+1/\lambda^2\right)$-robust.
\end{lemma}

\begin{proof}
Fix an LA-B service of length $g=2^j$ at time $t_0$.
Use its associated cover sequence and the resulting values $b$ and $d_b$ to construct a charge to $\OPT$.
Let $S$ contain the requests cleared at $t_0$ whose locations lie in $(2^b,2^j]$.
No request in $S$ is served by $\OPT$ before $t_0$.
Indeed, the claim is vacuous for $b=j$.
For $b<j$, every request $r=(\tau,x,w)\in S$ arrives after $t_k$, since the service at $t_k$ would otherwise clear it.
If $\OPT$ served $r$ at some $s<t_0$, that service would have length at least $x$ and would be large: its threshold is either $0$ or the length of an LA-B service in $[s,t_0)$, which must be smaller than $x$ because $r$ remains pending until $t_0$.
This contradicts $L=0$ when $b=-\infty$, and it contradicts $L<2^b<x$ when $b\in\mathbb Z$.

Tightness and the trigger invariant give
$$
Q(t_0,2^j)=2^{j-2}
\qquad\text{and}\qquad
Q(t_0,2^b)\le d_b.
$$
Here, the second relation follows by equality when $b=j$, by the trigger invariant when $b\in\mathbb Z$ and $b<j$, and by convention when $b=-\infty$.
Hence,
$$
\sum_{r\in S}q(t_0,r)
=Q(t_0,2^j)-Q(t_0,2^b)
\ge 2^{j-2}-d_b.
$$
Since $\OPT$ leaves each $r\in S$ unserved until $t_0$ and $q(t_0,r)\le w(t_0-\tau)/\lambda$, its delay cost on $S$ is at least
$$
\lambda\sum_{r\in S}q(t_0,r)
\ge \lambda(2^{j-2}-d_b).
$$
If $b\in\mathbb Z$, then $L\ge 2^{b-1}$, so $(s^\star,L)$ contains the half-open segment $(2^{b-2},2^{b-1}]$ of length $d_b$.
Adding this service-cost segment gives a total charge of at least
$$
\lambda(2^{j-2}-d_b)+d_b
=\lambda 2^{j-2}+(1-\lambda)d_b
\ge \lambda 2^{j-2}
=\frac{\lambda g}{4}.
$$

The delay charges are disjoint because LA-B clears each request once, and segments charged to different $\OPT$ services are distinct cost portions.
Suppose two LA-B services at times $u<v$, of lengths $2^j$ and $2^{j'}$, charge the same $\OPT$ service.
If $j'\le j$, the later cover sequence stops at or after $u$ and therefore excludes the shared $\OPT$ service, which occurs no later than $u$.
Hence $j'>j$.
Moreover, the interval containing $u$ in the later sequence has threshold at least $\ell(u)=2^j$.
Because thresholds do not decrease backward, the shared large service has length $L\ge 2^j$.
Consequently, the earlier charge uses $b=j$, whereas the later charge uses $b'\ge j+1$.
Their half-open dyadic segments are disjoint, proving that all charges are disjoint.

Lemma~\ref{lem:one-service-cost} bounds its LA-B cost by
$$
g+\frac{g}{4\lambda}.
$$
The charge constructed above has value at least $\lambda g/4$, so the ratio is at most
$$
\frac{g+g/(4\lambda)}{\lambda g/4}
=\frac4\lambda+\frac1{\lambda^2}.
$$
The service costs and the cleared requests' delay costs partition $\cost(\text{LA-B})$, while the associated charges sum to at most $\cost(\OPT)$.
Summing the displayed bound over all services proves the lemma.
\end{proof}

\paragraph{Consistency.}
We now compare LA-B with $\ADV$.
For each request $r=(\tau,x,w)$, we write $t_{\mathrm{ser}}(r)$ for the time when LA-B serves $r$.
We define
$$
\text{pre}(r)
=w\bigl(\min\{t_{\mathrm{ser}}(r),t_{\mathrm{cov}}(r)\}-\tau\bigr)_+,
$$
and
$$
\text{post}(r)
=w\bigl(t_{\mathrm{ser}}(r)-t_{\mathrm{cov}}(r)\bigr)_+.
$$
We also define $\mathrm{Pre}=\sum_r\text{pre}(r)$ and $\mathrm{Post}=\sum_r\text{post}(r)$.
These two parts satisfy $\delay(\text{LA-B})=\mathrm{Pre}+\mathrm{Post}$.
The first advice service that covers $r$ occurs at $t_{\mathrm{cov}}(r)$.
Thus, the delay paid by $\ADV$ for $r$ is at least $\text{pre}(r)$.
The sum over all requests gives $\mathrm{Pre}\le \delay(\ADV)$.

\begin{lemma}
\label{lem:post-advice-waiting}
$\mathrm{Post}\le \lambda\service(\ADV)$.
\end{lemma}

\begin{proof}
We prove the lemma by grouping requests according to their first coverage times.
Fix an advice time $s$ at which the advice serves length $a_s$.
Let $B_s$ be the set of requests with $t_{\mathrm{cov}}(r)=s$ that are still pending after LA-B processes the advice at time $s$ and just before it checks the trigger.
A request served before $s$ has zero post-coverage delay, so excluding it from $B_s$ loses nothing.
If $B_s=\varnothing$, this advice time contributes nothing to $\mathrm{Post}$.
We therefore assume that $B_s\ne\varnothing$ and define $m_s=\max_{r\in B_s}x$ and $G_s=2^{\lceil\log_2m_s\rceil}$.
The level $G_s$ is the smallest level that contains every request in $B_s$.
Since the advice service of length $a_s$ covers every request in $B_s$, we have
$$
m_s\le a_s\qquad\text{and}\qquad m_s\le G_s<2m_s\le 2a_s.
$$
For $r=(\tau,x,w)\in B_s$, define its post-coverage urgency while it remains pending by $u_s(t,r)=\frac{w}{\lambda}(t-s)$, $\forall t\ge s$.
Because $t_{\mathrm{cov}}(r)=s$, the urgency of a pending request $r\in B_s$ can be written as
$$
q(t,r)=\lambda w(s-\tau)+u_s(t,r)\ge u_s(t,r).
$$
At the service time of $r$, this quantity satisfies
$$
u_s(t_{\mathrm{ser}}(r),r)=\frac{\text{post}(r)}{\lambda}.
$$
Thus, it is enough to bound the total post-coverage urgency that LA-B clears from $B_s$.

Let $Q^-(t,g)$ denote the level potential after processing arrivals and advice coverage at time $t$, but immediately before the trigger service.
List, in chronological order, the LA-B services that clear at least one request from $B_s$ as $(t_1,h_1),\ldots,(t_k,h_k)$.
Let $R_i\subseteq B_s$ be the requests from $B_s$ cleared by the service $(t_i,h_i)$.
These sets partition $B_s$.

We first consider the services before the last one.
For every $i<k$, we have $h_i<G_s$.
Otherwise, the service of length $h_i$ would clear every remaining request in $B_s$, since all their locations are at most $m_s\le G_s$.
The relevant service lengths also increase strictly.
After a service of length $h_i$, each request in $B_s$ that remains pending has location larger than $h_i$.
Hence, any later service that clears one of these requests must have length larger than $h_i$.
The values $h_1,\ldots,h_{k-1}$ are therefore distinct dyadic lengths below $G_s$.
Writing $G_s=2^J$, we obtain
$$
\sum_{i=1}^{k-1}h_i<\sum_{j=-\infty}^{J-1}2^j=G_s.
$$
At time $t_i$, the served level $h_i$ is tight.
Moreover, the post-coverage urgency of the requests in $R_i$ is only part of their full urgency.
Therefore,
$$
\sum_{r\in R_i}u_s(t_i,r)\le \sum_{r\in R_i}q(t_i,r)\le Q^-(t_i,h_i)=\frac{h_i}{4}.
$$
Summing over the first $k-1$ relevant services gives
$$
\sum_{i=1}^{k-1}\sum_{r\in R_i}u_s(t_i,r)<\frac{G_s}{4}.
$$
It remains to control the requests cleared by the last relevant service.
For this step, we use the following trigger invariant: at every trigger check and for every dyadic level $g$,
$$
Q^-(t,g)\le\frac{g}{4}.
$$
To see this, observe that a newly arriving request has zero urgency, while first coverage changes the growth rate of its urgency but not its current value.
Thus, arrivals and coverage events do not create upward jumps in $Q(t,g)$.
Between events, $Q(t,g)$ grows continuously, and a service can only decrease it.
Once $Q(t,g)$ reaches $g/4$, the level $g$ is tight, so LA-B serves the largest tight level, whose length is at least $g$.
The potential therefore cannot cross the threshold before the trigger is processed.

Immediately before the last relevant service at time $t_k$, every request in $R_k$ is pending and lies in $[0,G_s]$.
The trigger invariant now gives
$$
\sum_{r\in R_k}u_s(t_k,r)\le Q^-(t_k,G_s)\le\frac{G_s}{4}.
$$
Combining the contributions before and at the last service, we obtain
$$
\frac{1}{\lambda}\sum_{r\in B_s}\text{post}(r)=\sum_{i=1}^k\sum_{r\in R_i}u_s(t_i,r)<\frac{G_s}{2}.
$$
Since $G_s<2a_s$, it follows that
$$
\sum_{r\in B_s}\text{post}(r)<\frac{\lambda G_s}{2}<\lambda a_s.
$$
Finally, every request with positive post-coverage delay satisfies $t_{\mathrm{ser}}(r)>t_{\mathrm{cov}}(r)$.
Such a request is still pending at its unique first coverage time and hence belongs to exactly one set $B_s$.
The sets $B_s$ therefore cover all positive contributions to $\mathrm{Post}$ without overlap.
Summing the per-group bound over all advice times yields
$$
\mathrm{Post}=\sum_s\sum_{r\in B_s}\text{post}(r)\le\lambda\sum_s a_s=\lambda\service(\ADV).
$$
This proves the lemma.
\end{proof}

\begin{lemma}[Consistency]
\label{lemma:lab-consistency}
LA-B is $(4+\lambda)$-consistent.
\end{lemma}

\begin{proof}
We first bound the service cost of LA-B.
For each LA-B service of length $g$ at time $t$, tightness gives $g=4Q(t,g)$.
The service clears exactly the pending requests that contribute to $Q(t,g)$.
Each request contributes at its own service time only once.
Thus,
$$
\service(\text{LA-B})
=4\sum_r q(t_{\mathrm{ser}}(r),r).
$$
The definition of request urgency gives
$$
q(t_{\mathrm{ser}}(r),r)
=\lambda\text{pre}(r)+\frac{1}{\lambda}\text{post}(r).
$$
Therefore,
\begin{align*}
\service(\text{LA-B})&=4\lambda\mathrm{Pre}+\frac{4}{\lambda}\mathrm{Post}\\
&\le 4\lambda\delay(\ADV)+4\service(\ADV).
\end{align*}
The last inequality uses $\mathrm{Pre}\le\delay(\ADV)$ and Lemma~\ref{lem:post-advice-waiting}.

The delay identity and the same two bounds give
\begin{align*}
\delay(\text{LA-B})
&=\mathrm{Pre}+\mathrm{Post}\\
&\le\delay(\ADV)+\lambda\service(\ADV).
\end{align*}
The two cost bounds now give
\begin{align*}
\MoveEqLeft \cost(\text{LA-B}) \\
&\le (4\lambda+1)\delay(\ADV) +(4+\lambda)\service(\ADV)\\
&\le (4+\lambda)\cost(\ADV).
\end{align*}
The last inequality uses $4\lambda+1\le 4+\lambda$ for $\lambda\in(0,1]$.
Hence, the lemma follows.
\end{proof}

\section*{Proof of Theorem 3}
Fix a given instance $I$.
When RB serves length $b_j$ at time $t$, tightness gives
$$
W(t,b_j)=\frac{b_j}{e}.
$$
The service cost is $b_j$.
The cleared requests have total delay cost $b_j/e$.
Thus, this step costs
$$
b_j+\frac{b_j}{e}=(e+1)\frac{b_j}{e}.
$$
We call $b_j/e$ the base charge of this service.
It is enough to prove
\begin{equation}
\mathbb E_U\!\left[
\sum_{\substack{\text{\rm RB services}\\(t,b_j)}}\frac{b_j}{e}
\right]
\le \cost(\OPT).
\label{eq:supp-rb-base-charge}
\end{equation}

We now fix the offset $U$ and one RB service $(t_0,b_j)$.
We use the cover sequence for \textsc{Balance}.
We write $\ell(t)$ for the length served by RB at time $t$.
We build times
$$
t_k<t_{k-1}<\cdots<t_1<t_0
$$
backward from $t_0$.
The time $t_1$ is the last RB service before $t_0$.
For each $i\ge 1$, the time $t_{i+1}$ is the last earlier service whose length is larger than $\ell(t_i)$.
The sequence stops at the first $t_k$ with $\ell(t_k)\ge b_j$.
An added zero-cost service of infinite length before the first arrival makes $t_k$ well defined.

We set the cover threshold to $0$ on $(t_1,t_0]$.
We set the cover threshold to $\ell(t_i)$ on each interval $(t_{i+1},t_i]$.
We call an $\OPT$ service in $(t_k,t_0]$ large when its length is at least the cover threshold at its service time.
When a large $\OPT$ service exists, we choose the latest one among the longest such services.
We write this service as $(s^\star,A)$.
When no large $\OPT$ service exists, we set $A=0$.

For $A>0$, we use the notation from the main text:
\begin{align*}
H_U(A)&=\min\{i\in\mathbb Z:b_i>A\},\\
h&=\min\{j,H_U(A)\}.
\end{align*}
For $A=0$, we set $h=-\infty$ and $b_h=0$.
We also set $W(t,b_h)=0$ in this case.
We split the base charge as
$$
\frac{b_j}{e}
=\frac{b_j-b_h}{e}+\frac{b_h}{e}.
$$
We charge the first part to $\OPT$ delay cost.
When $A>0$, we charge the second part to the service $(s^\star,A)$.

\begin{lemma}
\label{lem:no-early-opt-service}
$\OPT$ does not serve before $t_0$ any request that RB clears at $t_0$ and whose location lies in $(b_h,b_j]$.
\end{lemma}

\begin{proof}
The interval $(b_h,b_j]$ is empty when $h=j$.
We therefore take the case $h<j$.
We choose a request $r=(\tau,x,w)$ that RB clears at $t_0$ and that has $x\in(b_h,b_j]$.

The arrival time of $r$ satisfies $\tau>t_k$.
Otherwise, the service at $t_k$ would clear $r$ because $\ell(t_k)\ge b_j\ge x$.
For a contradiction, we assume that $\OPT$ serves $r$ at a time $s<t_0$.
We write the $\OPT$ service as $(s,L)$.
This service has $L\ge x$ and $s\in(t_k,t_0)$.

The service $(s,L)$ is large.
If $s\in(t_1,t_0]$, then its cover threshold is $0$.
If $s\in(t_{i+1},t_i]$, then the RB service at $t_i$ has length less than $x$.
This fact holds because $r$ remains pending for RB until $t_0$.
Thus, $L\ge x>\ell(t_i)$ in the second case.
Both cases show that $(s,L)$ is large.

The existence of $(s,L)$ gives $A>0$.
The choice of $(s^\star,A)$ also gives $A\ge L$.
Since $h<j$, the definition of $h$ gives $h=H_U(A)$.
The definition of $H_U(A)$ then gives $b_h>A$.
On the other hand, $L\ge x>b_h$.
The inequalities $A\ge L>b_h>A$ give a contradiction.
Hence, the lemma follows.
\end{proof}

Tightness gives
$$
W(t_0,b_j)=\frac{b_j}{e}.
$$
The level invariant also gives
$$
W(t_0,b_h)\le\frac{b_h}{e}.
$$
The delay values grow continuously between services.
Thus, a level cannot move above its tight value without starting a service.
The second bound uses our zero convention when $h=-\infty$.
Thus, the requests in $(b_h,b_j]$ have total delay at least
$$
W(t_0,b_j)-W(t_0,b_h)
\ge\frac{b_j-b_h}{e}.
$$
Lemma~\ref{lem:no-early-opt-service} shows that these requests remain pending for $\OPT$ through time $t_0$.
Thus, $\OPT$ pays at least this much delay on them.
Requests charged by different RB services are disjoint because RB clears each request once.
Therefore, the total delay charge is at most $\delay(\OPT)$.

\begin{lemma}
\label{lem:no-duplicate-level}
For fixed $U$, each $\OPT$ service receives at most one service charge at any level $h$.
\end{lemma}

\begin{proof}
We take two RB services $(t,b_j)$ and $(t',b_{j'})$ with $t<t'$.
We assume that both services charge the same $\OPT$ service $(s^\star,A)$.
We write their charged levels as $h$ and $h'$.
The earlier charge gives $s^\star\le t$.

We first show that $j'>j$.
If $j'\le j$, then the later cover sequence stops at or after $t$.
This sequence then cannot contain $s^\star\le t$.
This is impossible because the later service also charges $(s^\star,A)$.
Thus, $j'>j$.

The cover threshold at $s^\star$ in the later sequence is at least $b_j$.
This bound holds because the earlier RB service lies between $s^\star$ and $t'$.
It also holds because cover thresholds do not decrease as the sequence moves backward.
The service $(s^\star,A)$ is large for the later sequence.
Thus, $A\ge b_j$.

We write $m=H_U(A)$.
The inequalities $A\ge b_j$ and $b_m>A$ give $m>j$.
The earlier service uses level
$$
h=\min\{j,m\}=j.
$$
The later service uses level
$$
h'=\min\{j',m\}>j=h.
$$
Thus, the two charges use different levels.
The same argument applies to every pair of charges.
Hence, the lemma follows.
\end{proof}

For a fixed $U$, Lemma~\ref{lem:no-duplicate-level} bounds the total charge to an $\OPT$ service of length $A>0$ by
$$
\sum_{i\le H_U(A)}\frac{b_i}{e}.
$$
Every charged level is at most $H_U(A)$.
The lemma also shows that each level appears at most once.

\begin{lemma}
\label{lem:random-bidding}
For every fixed $A>0$, the following identity holds:
$$
\mathbb E_U\!\left[
\sum_{i\le H_U(A)}\frac{b_i}{e}
\right]=A.
$$
\end{lemma}

\begin{proof}
We write $\ell=\ln A$ and
$$
Z=H_U(A)+U-\ell.
$$
The index $H_U(A)$ is the first integer $i$ with $i+U>\ell$.
Thus, $Z$ is uniform on $[0,1]$, apart from an event of probability zero.
The first level above $A$ satisfies $b_{H_U(A)}=Ae^Z$.
The levels form a geometric sequence with ratio $e$.
Therefore,
$$
\sum_{i\le H_U(A)}\frac{b_i}{e}
=\frac{b_{H_U(A)}}{e-1}.
$$
We now take the expectation:
\begin{align*}
\mathbb E_U\!\left[
\sum_{i\le H_U(A)}\frac{b_i}{e}
\right]
&=\frac{A}{e-1}\int_0^1 e^z\,dz=A.
\end{align*}
Hence, the lemma follows.
\end{proof}

\begin{proof}[Proof of Theorem 3]
The delay-charge bound holds for every fixed offset $U$.
It gives a total delay charge of at most $\delay(\OPT)$.
Lemmas~\ref{lem:no-duplicate-level} and~\ref{lem:random-bidding} give an expected service charge of at most $A$ for each $\OPT$ service of length $A$.
The sum over all $\OPT$ services is at most $\service(\OPT)$.
The two bounds give
\begin{align*}
\mathbb E_U\!\left[
\sum_{\substack{\text{\rm RB services}\\(t,b_j)}}\frac{b_j}{e}
\right]
&\le\delay(\OPT)+\service(\OPT)\\
&=\cost(\OPT).
\end{align*}
This bound proves~\eqref{eq:supp-rb-base-charge}.

Each RB service costs $e+1$ times its base charge.
Therefore,
\begin{align*}
\mathbb E_U[\cost(\mathrm{RB})]
&=(e+1)\mathbb E_U\!\left[
\sum_{\substack{\text{\rm RB services}\\(t,b_j)}}\frac{b_j}{e}
\right]\\
&\le(e+1)\cost(\OPT).
\end{align*}
Hence, the theorem follows.
\end{proof}

\section*{Proof of Theorem 4}
By Yao's principle~\cite{yao1977probabilistic}, it suffices to construct a distribution over instances under which the expected cost of every deterministic online algorithm is at least $(e-o(1))$ times the expected offline optimum.
Fix $\varepsilon\in(0,1]$, an integer $N$ satisfying $(N+1)\varepsilon\geq 1$, and $R>\sqrt e$.
Let $K_1,\ldots,K_M$ be independent random variables taking values in $\{0,\ldots,N\}$ and satisfying
\begin{equation*}
\Pr[K_p\geq i]=e^{-i\varepsilon},\qquad i=0,\ldots,N.
\end{equation*}
Equivalently, $\Pr[K_p=k]=e^{-k\varepsilon}(1-e^{-\varepsilon})$ for $k<N$ and $\Pr[K_p=N]=e^{-N\varepsilon}$, so this is a well-defined distribution.
Starting from $T_1=0$, define $L_p=R^{2K_p}$ and $T_{p+1}=T_p+L_p$ for $p=1,\ldots,M$.
The instance $I$ consists of the $M$ consecutive intervals $[T_p,T_{p+1}]$.
At time $T_p$, requests
\begin{equation*}
\begin{aligned}
r_i^p&=(T_p,x_i^p,w_i^p),&x_i^p&=e^{i\varepsilon},\\
w_i^p&=R^{1-2i},&i&=0,\ldots,N,
\end{aligned}
\end{equation*}
arrive.
Associate with $r_i^p$ the relative artificial deadline $d_i^p=R^{2i}/2$.
Because $R^2>e>2$ and $L_p=R^{2K_p}$, the time $T_p+d_i^p$ lies strictly inside the $p$-th interval if and only if $i\leq K_p$.

We first consider a deterministic algorithm $\mathcal{A}$ with the following property.
\begin{property}\label{appendixprop:deadline}
For each interval $[T_p,T_{p+1}]$ and each $j\leq K_p$, algorithm $\mathcal{A}$ serves $r_j^p$ no later than time $T_p+d_j^p$.
\end{property}
Let $G_{N,\varepsilon}=e^{1-\varepsilon}(N+1)\varepsilon$.
\begin{lemma}
If deterministic algorithm $\mathcal{A}$ satisfying Property~\ref{appendixprop:deadline}, then $\mathbb{E}[\cost(\mathcal{A})]\geq MG_{N,\varepsilon}$.
\end{lemma}

\begin{proof}
To prove the claim, fix an interval $p$ and condition on the complete history $\mathcal{H}_p$ before $T_p$.
The variable $K_p$ remains independent of this history.
Consider the execution in which $K_p=N$, and order the services performed in the $p$-th interval chronologically, breaking ties by their execution order.
We select a sequence of services and a sequence of integer blocks as follows.
Set $a_1=0$.
Given $a_s$, let the $s$-th selected service be the first service in this execution whose length $\ell_s$ is at least $x_{a_s}^p$, and define $b_s=\max\{i\in\{0,\ldots,N\}:x_i^p\leq\ell_s\}$.
If $b_s=N$, stop; otherwise, set $a_{s+1}=b_s+1$ and continue.
Property~\ref{appendixprop:deadline} guarantees that the $s$-th selected service exists and occurs by time $T_p+d_{a_s}^p$.
Moreover, the selected services are distinct, and the blocks $[a_s,b_s]$ form a partition of $\{0,\ldots,N\}$.

For every realization with $K_p\geq a_s$, we have
\begin{equation*}
d_{a_s}^p=\frac{R^{2a_s}}{2}<R^{2a_s}\leq R^{2K_p}=L_p.
\end{equation*}
Hence no new request batch arrives before the time at which the $s$-th selected service is performed.
The input observed up to that time is therefore identical to the input in the execution with $K_p=N$.
Because $\mathcal{A}$ is deterministic, the same service of length $\ell_s$ is performed whenever $K_p\geq a_s$.
Consequently, if $\service_p(\mathcal{A})$ denotes the total service cost incurred strictly before $T_{p+1}$ in the $p$-th interval, then
\begin{align*}
\mathbb{E}[\service_p(\mathcal{A})\mid\mathcal{H}_p]
&\geq\sum_s\Pr[K_p\geq a_s]\ell_s\\
&\geq\sum_s e^{-a_s\varepsilon}e^{b_s\varepsilon}=\sum_s e^{(b_s-a_s)\varepsilon}.
\end{align*}
For $n_s=b_s-a_s+1$, the elementary inequality $e^z\geq ez$ for $z>0$ gives $e^{(b_s-a_s)\varepsilon}=e^{-\varepsilon}e^{n_s\varepsilon}\geq e^{1-\varepsilon}n_s\varepsilon$.
Since $\sum_s n_s=N+1$, the preceding two bounds imply
\begin{equation*}
\mathbb{E}[\service_p(\mathcal{A})\mid\mathcal{H}_p]\geq e^{1-\varepsilon}(N+1)\varepsilon=G_{N,\varepsilon}.
\end{equation*}
Removing the conditioning and summing over $p$ proves $\mathbb{E}[\cost(\mathcal{A})]\geq MG_{N,\varepsilon}$, because the intervals are disjoint and total cost is at least total service cost.
\end{proof}

We now prove Property~\ref{appendixprop:deadline} as follows.
\begin{lemma}\label{lem:cost-A}
For any deterministic algorithm $\mathcal{A}$, 
$$
\mathbb{E}[\cost(\mathcal{A})]\geq\frac{RM}{R+2e^{N\varepsilon}}G_{N,\varepsilon}.
$$
\end{lemma}
\begin{proof}
Let $E_p$ be the event that $\mathcal{A}$ misses at least one artificial deadline that lies in interval $p$.
Write $\rho_p=\Pr[E_p]$.
If $\mathcal{A}$ misses the deadline of $r_j^p$, then this request alone incurs delay cost at least $w_j^pd_j^p=R^{1-2j}R^{2j}/2=R/2$.
Choosing one missed request from each interval in which $E_p$ occurs and using the additivity of delay costs, we obtain
\begin{equation}\label{appendixeq:dcost-A}
\mathbb{E}[\cost(\mathcal{A})]\geq\mathbb{E}[\delay(\mathcal{A})]\geq\frac{R}{2}\sum_{p=1}^M\rho_p.
\end{equation}

Construct an auxiliary online algorithm $\mathcal{A}'$ that runs an internal simulation of $\mathcal{A}$ and copies all of its services.
At the first artificial deadline missed by the simulated algorithm in interval $p$, algorithm $\mathcal{A}'$ additionally performs a service of length $x_N^p=e^{N\varepsilon}$.
At such a deadline the interval has not ended, so this action uses only information available online.
The additional service covers every request released at $T_p$; hence $\mathcal{A}'$ performs at most one additional service per interval and satisfies Property~\ref{appendixprop:deadline}.
Its additional services can only reduce delay cost, and thus, for every realization,
\begin{align*}
\cost(\mathcal{A}')\leq{}&\cost(\mathcal{A})+e^{N\varepsilon}\sum_{p=1}^M\mathbf{1}_{E_p}.
\end{align*}
Taking expectations and applying the preceding lower bound to $\mathcal{A}'$ gives
\begin{align*}
\mathbb{E}[\cost(\mathcal{A})]+e^{N\varepsilon}\sum_{p=1}^M\rho_p
\geq\mathbb{E}[\cost(\mathcal{A}')]
\geq MG_{N,\varepsilon}.
\end{align*}
Let $C=\mathbb{E}[\cost(\mathcal{A})]$.
Inequality~\eqref{appendixeq:dcost-A} implies $\sum_p\rho_p\leq 2C/R$.
Substituting this bound into the last display yields
\begin{equation}%\label{eq:cost-A}
\mathbb{E}[\cost(\mathcal{A})] \geq \frac{RM}{R+2e^{N\varepsilon}} \cdot G_{N,\varepsilon},
\end{equation}
which proves this lemma.
\end{proof}

It remains to upper-bound the expected offline OPT.
\begin{lemma}\label{lem:cost-opt}
$\mathbb{E}[\cost(\OPT)]\leq MU_{N,\varepsilon,R}+e^{N\varepsilon}$, where $U_{N,\varepsilon,R}=1+N(1-e^{-\varepsilon})+NR\frac{1-e^{-\varepsilon}}{e^{-\varepsilon}R^2-1}$.
\end{lemma}

\begin{proof}
Consider the offline algorithm $\mathcal{B}$ that performs a service of length $x_{K_p}^p$ at time $T_p$ for every $p=1,\ldots,M$, and performs a final service of length $x_N^M$ at time $T_{M+1}$.
The final service makes $\mathcal{B}$ feasible.
For any random variable $K$ distributed as the $K_p$'s, the tail-sum identity gives
\begin{align*}
\mathbb{E}[e^{\varepsilon K}]
&=1+\sum_{i=1}^N\bigl(e^{i\varepsilon}-e^{(i-1)\varepsilon}\bigr)\Pr[K\geq i]\\
&=1+\sum_{i=1}^N\bigl(e^{i\varepsilon}-e^{(i-1)\varepsilon}\bigr)e^{-i\varepsilon}=1+N(1-e^{-\varepsilon}).
\end{align*}
Therefore,
\begin{equation}\label{appendixeq:scost-B}
\mathbb{E}[\service(\mathcal{B})]=M+MN(1-e^{-\varepsilon})+e^{N\varepsilon}.
\end{equation}

We next bound the delay cost of $\mathcal{B}$.
Set $\theta=e^{-\varepsilon}$.
Fix $p$ and $j\in\{1,\ldots,N\}$, and let $D_{p,j}$ be the waiting time of request $r_j^p$ under $\mathcal{B}$.
At time $T_q$, this request is served if $K_q\geq j$; otherwise, it remains pending throughout interval $q$ and accumulates an additional delay of $L_q$.
The final service at $T_{M+1}$ therefore gives the identity
\begin{equation*}
D_{p,j}=\sum_{q=p}^M R^{2K_q}\mathbf{1}\{K_h<j\text{ for every }h=p,\ldots,q\}.
\end{equation*}
Let $A_j=\mathbb{E}[R^{2K}\mathbf{1}_{\{K<j\}}]=(1-\theta)\sum_{k=0}^{j-1}(\theta R^2)^k$.
Since $\Pr[K<j]=1-\theta^j$ and the variables $K_p$ are independent, the waiting-time identity implies
\begin{align*}
\mathbb{E}[D_{p,j}]
&=A_j\sum_{q=p}^M(1-\theta^j)^{q-p}\leq\frac{A_j}{\theta^j}.
\end{align*}
The assumptions $R>\sqrt e$ and $\varepsilon\leq 1$ imply $\theta R^2>1$.
Using the geometric-sum formula for $A_j$, we obtain
\begin{align*}
w_j^p\mathbb{E}[D_{p,j}]
&\leq R^{1-2j}\theta^{-j}(1-\theta)\frac{(\theta R^2)^j-1}{\theta R^2-1}\\
&=\frac{R(1-\theta)}{\theta R^2-1}\left(1-(\theta R^2)^{-j}\right)\leq\frac{R(1-\theta)}{\theta R^2-1}.
\end{align*}
The request $r_0^p$ is served immediately because $K_p\geq0$ with probability one.
Summing the last bound over the remaining $N$ requests in each of the $M$ intervals yields
\begin{equation}\label{appendixeq:dcost-B}
\mathbb{E}[\delay(\mathcal{B})]\leq MNR\frac{1-e^{-\varepsilon}}{e^{-\varepsilon}R^2-1}.
\end{equation}
Since $\OPT$ is no more expensive than the feasible solution $\mathcal{B}$, equations~\eqref{appendixeq:scost-B} and~\eqref{appendixeq:dcost-B} give
\begin{equation*}
U_{N,\varepsilon,R}=1+N(1-e^{-\varepsilon})+NR\frac{1-e^{-\varepsilon}}{e^{-\varepsilon}R^2-1}
\end{equation*}
and
\begin{equation*}
\mathbb{E}[\cost(\OPT)]\leq MU_{N,\varepsilon,R}+e^{N\varepsilon},
\end{equation*}
which proves this lemma.
\end{proof}

\begin{proof}[Proof of Theorem 4]
Combining Lemma~\ref{lem:cost-A} and Lemma~\ref{lem:cost-opt}, Yao's principle implies that the competitive ratio of any randomized online algorithm against an oblivious adversary is at least
\begin{equation*}
\frac{RM G_{N,\varepsilon}}{(R+2e^{N\varepsilon})(MU_{N,\varepsilon,R}+e^{N\varepsilon})}.
\end{equation*}
For fixed $N$ and $\varepsilon$, first letting $M\to\infty$ and then $R\to\infty$ makes the last display converge to
\begin{equation}%\label{eq:lower-bound}
\frac{e^{1-\varepsilon}(N+1)\varepsilon}{1+N(1-e^{-\varepsilon})}.
\end{equation}
Letting $N\to\infty$ and then $\varepsilon\to0$ gives
\begin{equation*}
\lim_{\varepsilon\to0}\lim_{N\to\infty}\frac{e^{1-\varepsilon}(N+1)\varepsilon}{1+N(1-e^{-\varepsilon})}
=\lim_{\varepsilon\to0}\frac{e^{1-\varepsilon}\varepsilon}{1-e^{-\varepsilon}}
=e.
\end{equation*}
Thus, for every $c<e$, the parameters can be chosen so that the distribution above forces expected competitive ratio greater than $c$ for every deterministic algorithm.
Thus no randomized online algorithm has competitive ratio below $e$ against an oblivious adversary.
\end{proof}

\section*{Proof of Theorem 5}
Fix a given instance $I$ and a feasible advice $\ADV$.
The oblivious adversary chooses both sequences before the random offset $U$.
The proof combines the LA-B charging argument with the RB random-grid argument.

For every pending request $r=(\tau,x,w)$, the request urgency satisfies
$$
\lambda w(t-\tau)
\le q(t,r)
\le \frac{1}{\lambda}w(t-\tau).
$$
The definition of $q(t,r)$ gives these bounds.
The pre-coverage part grows at rate $\lambda w$.
The post-coverage part grows at rate $w/\lambda$.

\paragraph{Robustness.}
We first bound the cost of one LA-RB service.

\begin{lemma}
\label{lem:random-one-service-cost}
If LA-RB serves length $b_j$ at time $t$, then its local cost is at most $b_j+b_j/(e\lambda)$.
\end{lemma}

\begin{proof}
Tightness gives $Q(t,b_j)=\frac{b_j}{e}$.
The service clears exactly the pending requests with location at most $b_j$.
The lower urgency bound gives
$$
\sum_{\substack{r\text{ cleared}\\\text{at time }t}}w(t-\tau)
\le \frac{1}{\lambda}
\sum_{\substack{r\text{ cleared}\\\text{at time }t}}q(t,r)
=\frac{b_j}{e\lambda}.
$$
The service itself costs $b_j$.
Hence, the lemma follows.
\end{proof}

We call $\lambda b_j/e$ the robust base charge of this service.
We next charge this amount to $\OPT$.
We fix the offset $U$ and one LA-RB service $(t_0,b_j)$.
We write $\ell(t)$ for the length served by LA-RB at time $t$.
We build a cover sequence
$$
t_k<t_{k-1}<\cdots<t_1<t_0
$$
backward from $t_0$.
The time $t_1$ is the last LA-RB service before $t_0$.
For each $i\ge1$, the time $t_{i+1}$ is the last earlier service whose length is larger than $\ell(t_i)$.
The sequence stops at the first $t_k$ with $\ell(t_k)\ge b_j$.
An added zero-cost service of infinite length before the first arrival makes $t_k$ well defined.

We set the cover threshold to $0$ on $(t_1,t_0]$.
We set the cover threshold to $\ell(t_i)$ on each interval $(t_{i+1},t_i]$.
We call an $\OPT$ service in $(t_k,t_0]$ large when its length is at least the cover threshold at its service time.
When a large $\OPT$ service exists, we choose the latest one among the longest such services.
We write this service as $(s^\star,A)$.
When no large $\OPT$ service exists, we set $A=0$.

For $A>0$, we use
\begin{align*}
H_U(A)&=\min\{i\in\mathbb Z:b_i>A\},\\
h&=\min\{j,H_U(A)\}.
\end{align*}
For $A=0$, we set $h=-\infty$, $b_h=0$, and $Q(t,b_h)=0$.
We split the robust base charge as
$$
\lambda\frac{b_j}{e}
=\lambda\frac{b_j-b_h}{e}
+\lambda\frac{b_h}{e}.
$$

\begin{lemma}
\label{lem:learning-no-early-opt-service}
$\OPT$ does not serve before $t_0$ any request that LA-RB clears at $t_0$ and whose location lies in $(b_h,b_j]$.
\end{lemma}

\begin{proof}
The interval $(b_h,b_j]$ is empty when $h=j$.
We therefore take the case $h<j$.
We choose a request $r=(\tau,x,w)$ that LA-RB clears at $t_0$ and that has $x\in(b_h,b_j]$.

The arrival time of $r$ satisfies $\tau>t_k$.
Otherwise, the service at $t_k$ would clear $r$ because $\ell(t_k)\ge b_j\ge x$.
For a contradiction, we assume that $\OPT$ serves $r$ at a time $s<t_0$.
We write this $\OPT$ service as $(s,L)$.
This service has $L\ge x$ and $s\in(t_k,t_0)$.

The service $(s,L)$ is large.
If $s\in(t_1,t_0]$, then its cover threshold is $0$.
If $s\in(t_{i+1},t_i]$, then the LA-RB service at $t_i$ has length less than $x$.
This fact holds because $r$ remains pending for LA-RB until $t_0$.
Thus, $L\ge x>\ell(t_i)$ in the second case.
Both cases show that $(s,L)$ is large.

The existence of $(s,L)$ gives $A>0$.
The choice of $(s^\star,A)$ also gives $A\ge L$.
Since $h<j$, the definition of $h$ gives $h=H_U(A)$.
The definition of $H_U(A)$ then gives $b_h>A$.
On the other hand, $L\ge x>b_h$.
The inequalities $A\ge L>b_h>A$ give a contradiction.
This contradiction proves the lemma.
\end{proof}

The request urgencies grow continuously between events.
They have no upward jumps at arrival or coverage times.
Thus, no level can move above its tight value without starting a service.
Tightness and the level invariant give
$$
Q(t_0,b_j)=\frac{b_j}{e}
\qquad\text{and}\qquad
Q(t_0,b_h)\le\frac{b_h}{e}.
$$
The second bound uses the zero convention when $h=-\infty$.
Thus, the requests in $(b_h,b_j]$ have total urgency at least
$$
Q(t_0,b_j)-Q(t_0,b_h)
\ge\frac{b_j-b_h}{e}.
$$
Lemma~\ref{lem:learning-no-early-opt-service} shows that these requests remain pending for $\OPT$ through time $t_0$.
The upper urgency bound shows that their real delay is at least $\lambda$ times their urgency.
We can therefore charge $\lambda(b_j-b_h)/e$ to $\OPT$ delay cost.
When $A>0$, we charge the remaining amount $\lambda b_h/e$ to the service $(s^\star,A)$.

\begin{lemma}
\label{lem:learning-no-duplicate-level}
For fixed $U$, each $\OPT$ service receives at most one LA-RB service charge at any level $h$.
\end{lemma}

\begin{proof}
We take two LA-RB services $(t,b_j)$ and $(t',b_{j'})$ with $t<t'$.
We assume that both services charge the same $\OPT$ service $(s^\star,A)$.
We write their charged levels as $h$ and $h'$.
The earlier charge gives $s^\star\le t$.

We first show that $j'>j$.
If $j'\le j$, then the later cover sequence stops at or after $t$.
This sequence then cannot contain $s^\star\le t$.
This is impossible because the later service also charges $(s^\star,A)$.
Thus, $j'>j$.

The cover threshold at $s^\star$ in the later sequence is at least $b_j$.
This bound holds because the earlier LA-RB service lies between $s^\star$ and $t'$.
It also holds because cover thresholds do not decrease as the sequence moves backward.
The service $(s^\star,A)$ is large for the later sequence.
Thus, $A\ge b_j$.

We write $m=H_U(A)$.
The inequalities $A\ge b_j$ and $b_m>A$ give $m>j$.
The earlier service uses level $h=\min\{j,m\}=j$.
The later service uses level $h'=\min\{j',m\}>j=h$.
Thus, the two charges use different levels.
The same argument applies to every pair of charges.
Hence, the lemma follows.
\end{proof}

\begin{lemma}
\label{lem:learning-expected-base-charge}
For every fixed instance and advice sequence, the robust base charges satisfy
$$
\mathbb E_U\!\left[
\sum_{\substack{\text{\rm LA-RB services}\\(t,b_j)}}
\lambda\frac{b_j}{e}
\right]
\le\cost(\OPT).
$$
\end{lemma}

\begin{proof}
Requests used by different delay charges are disjoint because LA-RB clears each request once.
Thus, the total delay charge is at most $\delay(\OPT)$ for every fixed $U$.

We now take one $\OPT$ service of length $A>0$.
Lemma~\ref{lem:learning-no-duplicate-level} bounds its service charge for fixed $U$ by $\lambda\sum_{i\le H_U(A)}\frac{b_i}{e}$.
Lemma~\ref{lem:random-bidding} gives
$$
\mathbb E_U\!\left[
\lambda\sum_{i\le H_U(A)}\frac{b_i}{e}
\right]
=\lambda A
\le A.
$$
The sum over all $\OPT$ services is at most $\service(\OPT)$ in expectation.
The delay and service bounds prove the lemma.
\end{proof}

\begin{lemma}[Robustness]
\label{lemma:learning-randomized-robustness}
LA-RB is $(e/\lambda+1/\lambda^2)$-robust against an oblivious adversary.
\end{lemma}

\begin{proof}
Lemma~\ref{lem:random-one-service-cost} gives
\begin{align*}
b_j+\frac{b_j}{e\lambda}
&=\left(e+\frac{1}{\lambda}\right)\frac{b_j}{e}
=\left(\frac{e}{\lambda}+\frac{1}{\lambda^2}\right)
\lambda\frac{b_j}{e}.
\end{align*}
Lemma~\ref{lem:learning-expected-base-charge} bounds the expected sum of the robust base charges by $\cost(\OPT)$.
Thus,
$$
\mathbb E_U[\cost(\text{LA-RB})]
\le
\left(\frac{e}{\lambda}+\frac{1}{\lambda^2}\right)
\cost(\OPT).
$$
Hence, the lemma follows.
\end{proof}

\paragraph{Consistency.}
We now compare LA-RB with the feasible advice $\ADV$.
For each request $r=(\tau,x,w)$, we write $t_{\mathrm{ser}}(r)$ for the time when LA-RB serves $r$.
We define
$$
\text{pre}(r)
=w\bigl(\min\{t_{\mathrm{ser}}(r),t_{\mathrm{cov}}(r)\}-\tau\bigr)_+,
$$
and
$$
\text{post}(r)
=w\bigl(t_{\mathrm{ser}}(r)-t_{\mathrm{cov}}(r)\bigr)_+.
$$
We also define $\mathrm{Pre}=\sum_r\text{pre}(r)$ and $\mathrm{Post}=\sum_r\text{post}(r)$.
These two parts satisfy $\delay(\text{LA-RB})=\mathrm{Pre}+\mathrm{Post}$.
The first advice service that covers $r$ occurs at $t_{\mathrm{cov}}(r)$.
Thus, the delay paid by $\ADV$ for $r$ is at least $\text{pre}(r)$.
The sum over all requests gives $\mathrm{Pre}\le\delay(\ADV)$ for every fixed offset $U$.

\begin{lemma}
\label{lem:learning-post-advice-waiting}
$\mathbb E_U[\mathrm{Post}]
\le\lambda \service(\ADV).$
\end{lemma}

\begin{proof}
Recall that the random levels are $b_i=e^{i+U}$ for $i\in\mathbb Z$, where $U\sim\mathrm{Unif}[0,1]$.
We group requests according to their first coverage times.
Fix an advice time $s$ with advice length $a_s>0$.
For a fixed offset $U$, let $B_s$ be the set of requests with $t_{\mathrm{cov}}(r)=s$ that are still pending after LA-RB processes the advice at time $s$ and just before it checks the trigger.
Both $B_s$ and the service times of its requests may depend on $U$; we suppress this dependence until we take expectations.
A request served before $s$ has zero post-coverage delay, so excluding it from $B_s$ loses nothing.

Let $H_U(a_s)=\min\{i\in\mathbb Z:b_i>a_s\}$ and define $G_s=b_{H_U(a_s)}$.
Thus, $G_s$ is the first random level strictly above $a_s$.
Every request in $B_s$ has location at most $a_s<G_s$.
If $B_s=\varnothing$, its contribution to $\mathrm{Post}$ is zero.
We therefore fix $U$ and assume that $B_s\ne\varnothing$.

For a request $r=(\tau,x,w)\in B_s$, define its post-coverage urgency while it remains pending by
$$
u_s(t,r)=\frac{w}{\lambda}(t-s),\qquad t\ge s.
$$
Since $t_{\mathrm{cov}}(r)=s$, the urgency of a pending request $r\in B_s$ is
$$
q(t,r)=\lambda w(s-\tau)+u_s(t,r)\ge u_s(t,r).
$$
At the service time of $r$, we have
$$
u_s(t_{\mathrm{ser}}(r),r)=\frac{\text{post}(r)}{\lambda}.
$$
It is therefore enough to bound the total post-coverage urgency cleared from $B_s$.

Let $Q^-(t,b_i)$ denote the level potential after processing arrivals and advice coverage at time $t$, but immediately before the trigger service.
List the LA-RB services that clear at least one request from $B_s$ as $(t_1,h_1),\ldots,(t_k,h_k)$ in chronological order, where each $h_\ell$ is one of the random levels.
Let $R_\ell\subseteq B_s$ be the requests from $B_s$ cleared by the service $(t_\ell,h_\ell)$.
The sets $R_1,\ldots,R_k$ partition $B_s$.

We first consider the services before the last one.
For every $\ell<k$, we have $h_\ell<G_s$.
Otherwise, the service of length $h_\ell$ would clear every remaining request in $B_s$, since all their locations are smaller than $G_s$.
The relevant service lengths also increase strictly.
After a service of length $h_\ell$, each request in $B_s$ that remains pending has location larger than $h_\ell$.
Any later service that clears one of these requests must therefore have a larger length.
At time $t_\ell$, the served level $h_\ell$ is tight, and hence
$$
\sum_{r\in R_\ell}u_s(t_\ell,r)\le \sum_{r\in R_\ell}q(t_\ell,r)\le Q^-(t_\ell,h_\ell)=\frac{h_\ell}{e}.
$$
The levels below $G_s$ form a geometric sequence with ratio $e$.
Because the relevant pre-last services use a finite set of distinct levels below $G_s$, their total contribution is strictly smaller than
$$
\sum_{b_i<G_s}\frac{b_i}{e}=\frac{G_s}{e(e-1)}.
$$
It remains to bound the requests cleared by the last relevant service.
For every random level $b_i$, the trigger rule maintains the invariant
$$
Q^-(t,b_i)\le\frac{b_i}{e}.
$$
Indeed, a newly arriving request has zero urgency, while first coverage changes the growth rate of its urgency but not its current value.
Thus, arrivals and coverage events do not create upward jumps in $Q(t,b_i)$.
Between events, the potential grows continuously, and a service can only decrease it.
Once $Q(t,b_i)$ reaches $b_i/e$, the level $b_i$ is tight, so LA-RB serves the largest tight level, whose length is at least $b_i$.
The potential cannot cross the threshold before the trigger is processed.

Immediately before the last relevant service at time $t_k$, every request in $R_k$ is pending and lies in $[0,G_s]$.
The trigger invariant gives
$$
\sum_{r\in R_k}u_s(t_k,r)\le Q^-(t_k,G_s)\le\frac{G_s}{e}.
$$
Combining the pre-last and last-service contributions, we obtain
$$
\frac{1}{\lambda}\sum_{r\in B_s}\text{post}(r)=\sum_{\ell=1}^k\sum_{r\in R_\ell}u_s(t_\ell,r)<\frac{G_s}{e(e-1)}+\frac{G_s}{e}=\frac{G_s}{e-1}.
$$
Therefore, for every fixed $U$,
$$
\sum_{r\in B_s}\text{post}(r)\le\lambda\frac{G_s}{e-1}.
$$
By the geometric-series identity,
$$
\frac{G_s}{e-1}=\sum_{i\le H_U(a_s)}\frac{b_i}{e}.
$$
Lemma~\ref{lem:random-bidding}, applied with $A=a_s$, now gives
$$
\mathbb E_U\!\left[\frac{G_s}{e-1}\right]=a_s.
$$
Taking expectations in the fixed-$U$ bound yields
$$
\mathbb E_U\!\left[\sum_{r\in B_s}\text{post}(r)\right]\le\lambda a_s.
$$
For every fixed $U$, a request with positive post-coverage delay satisfies $t_{\mathrm{ser}}(r)>t_{\mathrm{cov}}(r)$.
Such a request is still pending at its unique first coverage time and hence belongs to exactly one set $B_s$.
The sets $B_s$ therefore cover all positive contributions to $\mathrm{Post}$ without overlap.
The advice sequence is fixed independently of $U$, so summing over all advice times and using linearity of expectation gives
\begin{align*}
\mathbb E_U[\mathrm{Post}]&=\sum_{s:a_s>0}\mathbb E_U\!\left[\sum_{r\in B_s}\text{post}(r)\right]\\
&\le\lambda\sum_{s:a_s>0}a_s=\lambda\service(\ADV).
\end{align*}
This proves the lemma.
\end{proof}

\begin{lemma}[Consistency]
\label{lemma:learning-randomized-consistency}
LA-RB is $(e+\lambda)$-consistent.
\end{lemma}

\begin{proof}
We first bound the service cost of LA-RB.
For each service of length $b_j$ at time $t$, tightness gives $b_j=eQ(t,b_j)$.
The service clears exactly the requests that contribute to $Q(t,b_j)$.
Each request contributes at its own service time only once.
Thus, for every fixed $U$,
$$
\service(\text{LA-RB})
=e\sum_r q(t_{\mathrm{ser}}(r),r).
$$
The definition of request urgency gives
$$
q(t_{\mathrm{ser}}(r),r)
=\lambda\text{pre}(r)
+\frac{1}{\lambda}\text{post}(r).
$$
Therefore,
\begin{align*}
\MoveEqLeft \mathbb E_U[\service(\text{LA-RB})] \\
&=e\lambda\cdot\mathbb E_U[\mathrm{Pre}]+\frac{e}{\lambda}\cdot\mathbb E_U[\mathrm{Post}]\\
&\le e\lambda\cdot\delay(\ADV)+e\cdot\service(\ADV).
\end{align*}
The last bound uses $\mathrm{Pre}\le\delay(\ADV)$ and Lemma~\ref{lem:learning-post-advice-waiting}.

The delay identity and the same two bounds give
$$
\mathbb E_U[\delay(\text{LA-RB})]
\le\delay(\ADV)+\lambda\service(\ADV).
$$
The two cost bounds give
\begin{align*}
\MoveEqLeft \mathbb E_U[\cost(\text{LA-RB})] \\ 
&\le(1+e\lambda) \delay(\ADV) +(e+\lambda)\service(\ADV)\\
&\le(e+\lambda)\cost(\ADV).
\end{align*}
The last step uses $1+e\lambda\le e+\lambda$ for $\lambda\in(0,1]$.
Hence, the lemma follows.
\end{proof}

\bibliography{aaai2027}

\let\maketitle\arxivOriginalMaketitle
\let\bibliography\arxivOriginalBibliography

\bibliography{aaai2027}

\end{document}